%% file: main.tex
\documentclass[11pt,letterpaper]{article}

\usepackage[T1]{fontenc}
\usepackage{amsmath,amssymb,amsthm,mathtools}
\usepackage{array,booktabs}
\usepackage{microtype}
\usepackage{graphicx}
\usepackage{algorithm}
\usepackage{algpseudocode}
\usepackage{placeins}
\usepackage{hyperref}
\hypersetup{hidelinks}
\usepackage{url}
\usepackage{amsfonts}  
\usepackage[margin=1in]{geometry}

\newcommand{\R}{\mathbb{R}}
\newcommand{\one}{\mathbf{1}}
\newcommand{\cC}{\mathcal{C}}
\newcommand{\Gap}{\operatorname{Gap}}
\newcommand{\pos}[1]{\left[#1\right]_{+}}
\newcommand{\diag}{\operatorname{diag}}
\newcommand{\Bcert}{\mathcal{B}_{\mathrm{Bell}}}
\newcommand{\Tbell}{T_{\mathrm{Bell}}}
\newcommand{\epsi}{\varepsilon}
\DeclareMathOperator{\MaxCut}{MaxCut}

\newcommand{\Sp}{\mathcal{S}}
\newcommand{\bin}{\mathrm{bin}}

\newcommand{\ten}[1]{\times10^{#1}}

\newtheorem{theorem}{Theorem}[section]
\newtheorem{proposition}[theorem]{Proposition}
\newtheorem{lemma}[theorem]{Lemma}
\newtheorem{corollary}[theorem]{Corollary}
\theoremstyle{definition}
\newtheorem{definition}[theorem]{Definition}
\newtheorem{assumption}[theorem]{Assumption}

\title{Bellman-Certified Rounding for Sparse Policy Deployment in MDPs}

\author{
  Zhaojun Peng \\
  \texttt{Hantsukipzj@gmail.com}
}
\date{}

\begin{document}
\maketitle

\begin{abstract}
\input{sections/00_abstract}
\end{abstract}

\input{sections/01_introduction}
\input{sections/02_preliminaries}
\input{sections/03_certificate}
\input{sections/04_tightness}
\input{sections/05_hardness}
\input{sections/06_calibration}
\input{sections/07_conclusion}

\FloatBarrier
\clearpage

\FloatBarrier
\bibliography{refs}
\bibliographystyle{plain}

\appendix
\input{sections/A_certificate_proofs}
\input{sections/B_tightness_proof}
\input{sections/C_complexity_proof}
\input{sections/D_calibration_protocol}

\end{document}

%% file: sections/00_abstract.tex
Continuous policy optimization may spread an update across many states, even
when deployment permits only a few complete state-level changes. We study how
much discounted return can be retained when continuous row mixtures are
rounded to sparse binary policies in finite MDPs. Policy-dependent visitation
couples the row edits, while long horizons make global curvature bounds
conservative. From $2d+2$ Bellman solves, we derive reusable envelopes that
support uniform and candidate-specific guarantees before rounding. A rank-two
rational representation of each exchange further permits weighted curvature
integration along the realized trajectory. We prove that linear dimension
dependence is unavoidable when the budget scales, and that exact global
curvature thresholding is hard. Candidate-specific bounds raise pre-rounding
certification coverage from $48.2\%$ to $74.1\%$ on the structured suite. At
$\gamma=0.95$, local integration lowers the median bound-to-loss ratio from
$402.3$ to $2.08$ on coupled instances.

%% file: sections/01_introduction.tex
\section{Introduction}
\label{sec:introduction}

Continuous policy optimization naturally distributes small updates across
many states. Deployment may instead require each modified state-level decision
to be implemented as a complete row change, with only $k$ such changes
affordable. We study the return lost when a continuous policy is mapped to
this sparse deployment class.

We consider $d$ editable rows in a finite discounted Markov decision process.
Each row retains its baseline action or adopts one proposed action. A binary
vector $y$ specifies these choices, giving the deployment problem
\[
\max_{y\in\{0,1\}^d,\,\one^\top y\le k}J(y)
\].
Continuous row mixtures belong to $\cC_k=\{x\in[0,1]^d:\one^\top x\le k\}$,
and $\cC_k^{\bin}=\cC_k\cap\{0,1\}^d$ is the deployable class. Their mismatch is
\begin{equation}
\Gap(\cC_k)=\max_{x\in\cC_k}J(x)-\max_{y\in\cC_k^{\bin}}J(y).
\label{eq:gap}
\end{equation}
Our goal is to convert a continuous candidate into a feasible binary policy
with a return guarantee. Changing one row alters later visits to other rows,
so the values of the edits do not decouple. At long horizons, the largest
curvature over the policy cube may also lie far from the rounding trajectory
and substantially overstate its loss.

Every fractional row mixture is a stationary randomized policy in an
auxiliary two-action MDP. Extremal Bellman values bracket both policy values
and row occupancies. The resulting sensitivity envelopes require $2d+2$
solves and give $J(y)\ge J(x)-B_{\mathrm{unif}}$ for every $x\in\cC_k$, with
$c_{d,k}\le k/4$. Once a candidate is known, its geometry sharpens the bound
without another Bellman solve. For the remaining long-horizon slack, we use
the rank-two rational form of each exchange and integrate curvature with the
interpolation weight that determines endpoint loss.

\paragraph{Advance certification and deployment.}
A uniform loss bound controls the relaxation gap over the policy family. For a
particular candidate, the condition $J(x)-B>J(0)$ certifies improvement over
the baseline before the discrete policy is constructed. The same Bellman
envelopes can be reused across candidates and budgets that share the baseline
and proposed rows. On the structured suite, candidate-specific bounds increase
pre-rounding certification coverage from $48.2\%$ to $74.1\%$. Rounding uses a
median of six return evaluations, compared with 26 for greedy, while their mean
normalized optimality gaps are $0.185\%$ and $0.178\%$. At $S=d=16$, validated
arithmetic makes certificate construction $14.8\times$ faster than the
exact-rational pipeline. The realized return can then be checked directly once
rounding is complete.

\subsection{Illustrative example}
\label{sec:motivating}

Consider \(m\) zero-reward states and one absorbing target with positive reward.
At each zero-reward state, the baseline action self-loops and the proposed
action moves to the target. The deployment budget is \(k=m/2\), with one
dummy editable coordinate at the target giving \(d=m+1\). The return separates
across the \(m\) relevant rows,
\[
J(x)=\frac1m\sum_{i=1}^m\frac{\gamma x_i}
{(1-\gamma)(1-\gamma+\gamma x_i)} .
\]
Each summand is increasing and concave. The continuous optimum assigns
\(x_i=1/2\) to every relevant row, whereas a deployable policy selects
\(m/2\) rows and leaves the others unchanged. The continuous-to-binary loss is
therefore explicit. Proposition~\ref{prop:tightness} evaluates both the gap
and its certificate in closed form.

\subsection{Related work}
\label{sec:related}

Pipage rounding originates in the classical exchange framework of
\cite{AgeevS04}. Recent work extends continuous relaxations and rounding
beyond standard submodularity, including one-sided-smooth objectives and
nonmonotone submodular maximization
\cite{GhadiriSS25,BuchbinderF25}. Here the objective is the discounted return
of an MDP, and its exchange curvature is induced by the full
policy-dependent dynamics.

Recent safe-policy-improvement work studies finite-data guarantees, partial
observability, scalable planning, reliable decision points, and learned models
\cite{WienhoeftEtAl23,SimaoEtAl23,CastelliniEtAl23,pmlr-v235-bianchi24b,
BianchiEtAl25,SharmaEtAl25,HamelDeleCourtEtAl25,DelgrangeEtAl26}.
These methods control improvement relative to a baseline under data or model
uncertainty. We address the return lost at the next stage of the pipeline,
when a continuous model-based update must satisfy a \(k\)-row deployment
constraint. Related work also studies stochastic optimization when the
deployed solution must be deterministic
\cite{pmlr-v235-montenegro24a}. The constraint is also related to sparse and
discrete optimization, where cardinality or binary structure is handled by
hard thresholding, block-coordinate methods, or continuous reformulations
\cite{yuan2020block,yuan24a,yuan2016binary}. Our starting point is an optimized
continuous policy whose conversion to sparse binary deployment must be
certified.

Recent formal and robust-MDP work studies parameter synthesis, structural
constraints, uncertainty-aware computation, and robust policy optimization
\cite{JungesEtAl24,pmlr-v242-rickard24a,DBLP:journals/corr/abs-2504-05965,
AzeemEtAl25,Heck2026,HartmannsEtAl26,SAP26b,SuilenEtAl24,pmlr-v235-lin24u,
MeggendorferEtAl25,pmlr-v267-wang25ei,LiKuhnSutter26,AbateEtAl26,
pmlr-v336-asadi26a,KumarEtAl23,ZhangHuLi24,pmlr-v235-chen24s,SunEtAl24}.
Complementary second-order policy-optimization work analyzes Hessian
information directly
\cite{pmlr-v238-maniyar24a,pmlr-v235-mu24b,pmlr-v286-khorasani25a,ZhangEtAl26}.
Our Bellman envelopes instead provide uniform first- and second-order
sensitivity bounds over the policy cube, which we specialize to sparse
deployment.

\subsection{Contributions}
\label{sec:contributions}

\paragraph{Certified deployment.}
Bellman envelopes and budget-aware rounding geometry yield uniform and
candidate-specific return guarantees. Algorithm~\ref{alg:round} constructs a
binary policy with at most $2(d-1)$ endpoint evaluations.

\paragraph{Curvature and trajectory bounds.}
A lower-bound family identifies the necessary dimension dependence, while an
NP-hardness result limits exact global curvature computation. For an executed
exchange, its rank-two structure gives a weighted trajectory bound that avoids
paying for curvature away from the path.

\paragraph{Empirical behaviour.}
The experiments separate pre-rounding certification, final policy quality,
long-horizon tightness, and computation cost. They quantify when candidate
geometry suffices, when trajectory integration is needed, and how validated
arithmetic compares with exact rational computation.

%% file: sections/02_preliminaries.tex
\section{Preliminaries}
\label{sec:preliminaries}

We use \(\R_+\) for the nonnegative reals and \(\one\) for the all-ones
vector. The vector \(e_i\in\R^S\) selects a state row, while
\(\epsi_i\in\R^d\) selects a policy coordinate. For \(a\in\R\), let
\(\pos{a}=\max\{a,0\}\). We write \(v=v^+-v^-\) for the componentwise
positive and negative parts of a vector. The operator \(\diag(x)\) forms a
diagonal matrix, and \(\odot\) denotes outward-rounded interval
multiplication. The cost of one discounted Bellman solve on \(S\) states is
denoted by \(\Tbell(S)\).

\subsection{Problem setting and assumptions}
\label{sec:setup}

Let \(\Sp=\{1,\ldots,S\}\) be a finite state space. Let
\(\gamma\in(0,1)\) be the discount, and let \(\mu\in\R_+^S\) satisfy
\(\one^\top\mu=1\). After relabeling, rows \(1,\ldots,d\) are editable, where
\(d\in\{1,\ldots,S\}\). Each editable row \(i\) has a baseline action and one
proposed action. The weight \(x_i\in[0,1]\) interpolates between them, giving
\begin{align}
P(x)&=P^0+\sum_{i=1}^d x_i\,e_i d_i^\top,
&r(x)&=r^0+\sum_{i=1}^d x_i\,e_i\Delta r_i,
&x&\in[0,1]^d .
\label{eq:affine-mdp}
\end{align}
Here \(d_i^\top\) is the difference between the two transition rows, and
\(\Delta r_i\) is the reward difference.

\begin{assumption}[Stochastic two-action rows]
\label{asm:model}
\(P^0\) is stochastic. For every editable row \(i\), both
\(P^0_{i\cdot}\) and \(P^0_{i\cdot}+d_i^\top\) are probability
distributions over \(\Sp\).
\end{assumption}

Assumption~\ref{asm:model} makes \(P(x)\) stochastic for every point in the
cube. Therefore
\begin{equation}
R(x)=(I-\gamma P(x))^{-1},
\qquad
V(x)=R(x)r(x),
\qquad
J(x)=\mu^\top V(x)
\label{eq:resolvent}
\end{equation}
are well defined \cite{Puterman94}. The function \(J\) is real analytic on a
neighbourhood of the cube. We use the budget sets in \eqref{eq:gap}, with
\(k\in\{0,\ldots,d\}\). Their relaxation gap is nonnegative and vanishes when
the continuous and binary optima coincide.

A second assumption concerns the arithmetic used to construct the bounds.

\begin{assumption}[Exact or validated arithmetic]
\label{asm:arithmetic}
Bellman solves use exact rational arithmetic or directed outward rounding
\cite{MooreKearfottCloud09}.
In the latter case, the stochastic input, Bellman inequalities, residuals, and
all downstream interval operations are enclosed. Return comparisons in
Algorithm~\ref{alg:round} are exact, or the selected endpoint is certified to
have return at least that of the alternative.
\end{assumption}

Ordinary floating-point solves with residual padding do not satisfy this
assumption. The distinction is consequential in
Section~\ref{sec:calibration}: double precision reports many zero bounds on the
same cases for which exact arithmetic reports none.

\subsection{Sensitivity identities and rounding quantities}
\label{sec:defects}

The rounding procedure uses coordinate increases and pairwise exchanges.
Differentiating \eqref{eq:resolvent} introduces three quantities,
\begin{align}
q_i(x)&=\mu^\top R(x)e_i,
&a_i(x)&=\Delta r_i+\gamma d_i^\top V(x),
&b_{ij}(x)&=\gamma d_i^\top R(x)e_j ,
\label{eq:primitives}
\end{align}
and the identities
\begin{align}
\partial_iJ(x)&=q_i(x)a_i(x),\label{eq:gradient-id}\\
H_{ii}(x)&=2q_i(x)b_{ii}(x)a_i(x),\label{eq:hessian-diag}\\
H_{ij}(x)&=q_j(x)b_{ji}(x)a_i(x)+q_i(x)b_{ij}(x)a_j(x),\quad i\ne j.
\label{eq:hessian-cross}
\end{align}
Here \(H(x)=\nabla^2J(x)\). These are classical Markov-chain sensitivity
identities \cite{Cao_1998,JMLR:v17:15-414}. Their factored form is important
because every factor is evaluated at the same policy \(x\).

\begin{definition}[Quantities controlling the rounding loss]
\label{def:defects}
Define the largest negative coordinate slope and the largest negative pairwise
exchange curvature as
\begin{align}
\tau&=\max_{i}\ \sup_{x\in[0,1]^d}\pos{-\partial_iJ(x)},
\label{eq:tau}\\
\beta&=\max_{i<j}\ \sup_{x\in[0,1]^d}
\pos{-(\epsi_i-\epsi_j)^\top H(x)(\epsi_i-\epsi_j)} .
\label{eq:beta}
\end{align}
When \(d=1\), the pairwise maximum is empty and \(\beta=0\).
\end{definition}

Thus every coordinate derivative is at least \(-\tau\), and every exchange
line has second derivative at least \(-\beta\). The next section derives
tractable upper bounds. Sections~\ref{sec:tightness} and \ref{sec:hardness}
then study their dimension dependence and the cost of evaluating \(\beta\)
exactly.

%% file: sections/03_certificate.tex
\section{Bellman bounds for rounding}
\label{sec:certificate}

This section turns Bellman value bounds into a constructive rounding
certificate. Lemma~\ref{lem:bellman-envelope} first bounds every policy in the
continuous family. Proposition~\ref{prop:sensitivity-certificate} propagates
these envelopes to the required derivatives, and
Theorem~\ref{thm:certificate} supplies the rounding guarantee.

\subsection{Bellman envelopes}
\label{sec:envelopes}

Consider an auxiliary discounted control problem with the two endpoint actions
at each editable row. Every unedited row retains its fixed action. Let
\(\underline V\) and \(\overline V\) be the fixed points of the min- and
max-Bellman operators, so
\(\underline V_s\le V^\pi_s\le\overline V_s\) for every stationary
randomized policy \(\pi\) and every state \(s\). For each editable
state \(j\), let \(\underline h_j\) and \(\overline h_j\) be the
corresponding min/max value functions when the stage reward is \(e_j\).

\begin{lemma}[Bellman envelope]
\label{lem:bellman-envelope}
Under Assumption~\ref{asm:model}, every \(x\in[0,1]^d\) satisfies
\begin{equation}
\underline V\le V(x)\le\overline V,
\qquad
\underline h_j\le R(x)e_j\le\overline h_j
\quad\text{for every }j\in\{1,\ldots,d\}.
\label{eq:envelopes}
\end{equation}
All inequalities are componentwise.
\end{lemma}

\begin{proof}[Proof of Lemma~\ref{lem:bellman-envelope}]
Every \(x\in[0,1]^d\) defines a stationary randomized policy over the two
endpoint actions at each editable row. Standard discounted Bellman optimality
therefore bounds its value componentwise,
$
\underline V\le V(x)\le\overline V.
$
Apply the same argument to the auxiliary control problem with stage reward
\(e_j\). The value of policy \(x\) in this problem is \(R(x)e_j\), which gives
$
\underline h_j\le R(x)e_j\le\overline h_j.
$
Both statements hold simultaneously for every \(x\) and editable \(j\).
\end{proof}

Splitting \(d_i=d_i^+-d_i^-\) propagates
Lemma~\ref{lem:bellman-envelope} to intervals for the sensitivity quantities
in \eqref{eq:primitives},
\begin{align}
Q_i={}&[\mu^\top\underline h_i,\ \mu^\top\overline h_i],
\label{eq:canonical-Q}\\
A_i={}&\Bigl[\Delta r_i+\gamma\{(d_i^+)^\top\underline V-(d_i^-)^\top\overline V\},\
\Delta r_i+\gamma\{(d_i^+)^\top\overline V-(d_i^-)^\top\underline V\}\Bigr],
\label{eq:canonical-A}\\
B_{ij}={}&\gamma\Bigl[(d_i^+)^\top\underline h_j-(d_i^-)^\top\overline h_j,\
(d_i^+)^\top\overline h_j-(d_i^-)^\top\underline h_j\Bigr].
\label{eq:canonical-B}
\end{align}

\subsection{Sensitivity certificate}
\label{sec:sensitivity-certificate}

Using Minkowski interval addition and outward-rounded multiplication
\(\odot\), define
\begin{align}
G_i&=Q_i\odot A_i,
&\widehat\tau&=\max_i\pos{-\inf G_i},
\label{eq:tauhat}\\
\mathcal H_{ii}&=2\,Q_i\odot B_{ii}\odot A_i,
&\mathcal H_{ij}&=(Q_j\odot B_{ji}\odot A_i)+(Q_i\odot B_{ij}\odot A_j),
\label{eq:Hintervals}\\
\underline S_{ij}&=\inf\mathcal H_{ii}+\inf\mathcal H_{jj}
-2\sup\mathcal H_{ij},
&\widehat\beta&=\max_{i<j}\pos{-\underline S_{ij}},
\label{eq:betahat}
\end{align}
with \(\widehat\beta=0\) when \(d=1\).

\begin{proposition}[Sensitivity certificate]
\label{prop:sensitivity-certificate}
Let Assumptions~\ref{asm:model} and \ref{asm:arithmetic} hold. For every
\(x\in[0,1]^d\) and all editable \(i,j\),
$
q_i(x)\in Q_i,
\quad
a_i(x)\in A_i,
\quad
b_{ij}(x)\in B_{ij}.
$
Consequently,
$
\tau\le\widehat\tau,
\qquad
\beta\le\widehat\beta.
$
\end{proposition}

\begin{proof}[Proof of Proposition~\ref{prop:sensitivity-certificate}]
Since \(\mu\ge0\), Lemma~\ref{lem:bellman-envelope} gives
\(q_i(x)=\mu^\top R(x)e_i\in Q_i\). For any componentwise box
\(\underline z\le z\le\overline z\), the exact extrema of
\(d_i^\top z=(d_i^+)^\top z-(d_i^-)^\top z\) are
$
(d_i^+)^\top\underline z-(d_i^-)^\top\overline z
\quad\text{and}\quad
(d_i^+)^\top\overline z-(d_i^-)^\top\underline z.
$
Taking \(z=V(x)\) and \(z=R(x)e_j\) proves
\(a_i(x)\in A_i\) and \(b_{ij}(x)\in B_{ij}\).

Equation~\eqref{eq:gradient-id} then gives
\(\partial_iJ(x)\in G_i\). Hence
$
\pos{-\partial_iJ(x)}
\le\pos{-\inf G_i}
\le\widehat\tau.
$
Taking the maximum over \(i\) and the supremum over \(x\) yields
\(\tau\le\widehat\tau\).
Equations~\eqref{eq:hessian-diag} and \eqref{eq:hessian-cross} give
\(H_{ii}(x)\in\mathcal H_{ii}\) and
\(H_{ij}(x)\in\mathcal H_{ij}\). Therefore
\[
(\epsi_i-\epsi_j)^\top H(x)(\epsi_i-\epsi_j)
=H_{ii}(x)+H_{jj}(x)-2H_{ij}(x)
\ge\underline S_{ij}.
\]
Taking negative parts, the pairwise maximum, and the supremum over \(x\)
gives \(\beta\le\widehat\beta\). Directed outward interval operations
preserve every inclusion. Exact rational arithmetic adds zero
numerical-roundoff width.
\end{proof}

\subsection{Certified rounding}
\label{sec:certificate-thm}

For $q_{d,k}:=\min\{k,\lfloor d/2\rfloor\}$, define the budget coefficient
$c_{d,k}:=q_{d,k}(d-q_{d,k})/(4d)$.
\begin{theorem}[Budget-aware certified rounding]
\label{thm:certificate}
Under Assumptions~\ref{asm:model} and \ref{asm:arithmetic}, Algorithm~\ref{alg:round}
returns $y\in\cC_k^{\bin}$ for every $x\in\cC_k$, with
$J(y)\ge J(x)-B_{\mathrm{unif}}$ and $\Gap(\cC_k)\le B_{\mathrm{unif}}$,
where $B_{\mathrm{unif}}:=\widehat\tau+c_{d,k}\widehat\beta$ for $k>0$
and $B_{\mathrm{unif}}:=0$ for $k=0$.
\end{theorem}
For comparison with the original experiments, we also report the legacy bound
$\Bcert:=\widehat\tau+(d-1)\widehat\beta/8$.
The new bound satisfies $B_{\mathrm{unif}}\le\Bcert$:
for $d\ge2$, $c_{d,k}\le d/16\le(d-1)/8$; for $d=1$, $\widehat\beta=0$.
Moreover, $c_{d,k}\le k/4$, so its coefficient stays bounded at fixed budget.

Algorithm~\ref{alg:round}, stated in
Appendix~\ref{app:certificate-proofs}, first raises coordinates until their
sum is integral. It then exchanges mass between fractional coordinates and
selects the better endpoint. Each exchange makes at least one coordinate
integral. At most \(d-1\) exchanges and \(2(d-1)\) return evaluations are
required. These endpoint evaluations are separate from the $2d+2$ Bellman
solves used to construct the reusable certificate. Every endpoint comparison
must be exact or certified under Assumption~\ref{asm:arithmetic}.

The model-only guarantee depends on the MDP and budget, but not on the
candidate. A zero bound therefore proves equality of the continuous and binary
optima. The proof uses $\Phi(z):=\sum_i z_i(1-z_i)$. At the integer-sum point
after Phase~1, $\ell=\mathbf1^\top z\le k$, and Cauchy--Schwarz gives
$\Phi(z)\le\ell(d-\ell)/d\le4c_{d,k}$. Telescoping the exchange-potential
decrease then charges at most $\widehat\beta\Phi(z)/4$ in Phase~2.
Appendix~\ref{app:certificate-proofs} gives the proof without additional solves.

The same Bellman intervals also yield a sharper guarantee once a particular
continuous candidate is available. Fix once and for all a deterministic
Phase-1 fill order (we use increasing coordinate index) and let \(z=z(x)\)
be the point obtained by raising coordinates in that order until
\(\mathbf 1^\top z=\lceil \mathbf 1^\top x\rceil\). Define
\begin{equation}
\widehat\tau_i:=\pos{-\inf G_i},
\qquad
F(z):=\{i:0<z_i<1\},
\qquad
\widehat\beta_{F(z)}:=
\max_{\substack{i<j\\ i,j\in F(z)}}\pos{-\underline S_{ij}},
\label{eq:candidate-local-quantities}
\end{equation}
with \(\widehat\beta_{F(z)}=0\) when \(|F(z)|<2\). Let
$\Phi(z):=\sum_{i=1}^d z_i(1-z_i)$, $\Delta_i:=z_i-x_i$
and define the candidate-specific pre-rounding certificate
$
B_{\mathrm{cand}}(x):=
\sum_{i=1}^d \Delta_i\widehat\tau_i
+\frac{\widehat\beta_{F(z)}}{4}\Phi(z).
$
\begin{proposition}[Candidate-specific pre-rounding certificate]
\label{prop:candidate-certificate}
Let Assumptions~\ref{asm:model} and \ref{asm:arithmetic} hold, and let
Algorithm~\ref{alg:round} use the fixed Phase-1 fill order above. For every
\(x\in\cC_k\), its output \(y\in\cC_k^{\bin}\) satisfies
$
J(y)\ge J(x)-B_{\mathrm{cand}}(x),
\quad
B_{\mathrm{cand}}(x)\le B_{\mathrm{unif}}\le\Bcert.
$
Thus \(B_{\mathrm{cand}}(x)\) is available once the continuous candidate is
known and before any pairwise endpoint comparison.
\end{proposition}

The sharper bound retains candidate geometry that
Theorem~\ref{thm:certificate} discards. For an exchange of coordinates \(i,j\)
over the maximal interval \([a,b]\),
Proposition~\ref{prop:sensitivity-certificate}
gives the pair-specific curvature bound
\(\widehat\beta_{ij}=\pos{-\underline S_{ij}}\), and the proof of
Theorem~\ref{thm:certificate} yields the one-step charge
\(\widehat\beta_{ij}(-a)b/2\). Moreover, every maximal exchange satisfies the
potential identity
$
\Phi(z^t)-\Phi(z^{t+1})=2(-a_t)b_t.
$
Since the fractional support can only shrink and the final point is binary,
\(\sum_t(-a_t)b_t=\Phi(z)/2\). This identity yields the second term of
$B_{\mathrm{cand}}$. The first term records coordinate-specific charges along
the deterministic fill path. Appendix~\ref{app:certificate-proofs} provides
the proof.

The reusable certificate is constructed once per instance. Computing the
Bellman envelopes and interval quantities takes \(2d+2\) Bellman solves,
\(O(d^2S)\) interval operations, and an \(O(d^2)\) scan. The total cost is
\((2d+2)\Tbell(S)+O(d^2S)\). Given a candidate \(x\),
\(B_{\mathrm{cand}}(x)\) uses the stored \(G_i\) and
\(\underline S_{ij}\), so no further Bellman solve is needed. The fill and
potential cost \(O(d)\) in total, while the support scan costs at most
\(O(d^2)\). This
model-only construction is most attractive when \(d\ll S\). When \(d=S\),
the number of Bellman solves is twice the number of states. Rounding is then
performed once for each candidate.

\subsection{Computation and interval conservatism}

The preceding operation count does not include rational bit growth. With
\(L\)-bit inputs, exact rational arithmetic has bit complexity polynomial in
\(S\), \(d\), and \(L\). Validated floating point instead adds an
\(O(S^2)\) residual pass to each solve. Its intervals have positive width,
while outward endpoints preserve
Proposition~\ref{prop:sensitivity-certificate}.

The interval calculation treats \(Q_i\), \(A_i\), and \(B_{ij}\) as independent
boxes and can therefore include endpoint combinations that no common \(x\)
realizes. The excess width vanishes when the Bellman envelopes are tight.
Proposition~\ref{prop:tightness} gives
\(\widehat\beta=\beta\) on the example family.

%% file: sections/04_tightness.tex
\section{Dimension dependence and long-horizon conservatism}
\label{sec:tightness}

At fixed $k$, the budget-aware coefficient $c_{d,k}$ remains bounded.
Linear dependence on dimension is nevertheless unavoidable when
$k=\Theta(d)$. The family from Section~\ref{sec:motivating} also distinguishes
exchange-count slack from curvature conservatism.

\begin{proposition}[Fixed-discount lower-bound family]
\label{prop:tightness}
For every even $m\ge2$ and $\gamma\in(0,1)$, the illustrative family is a
transition-only MDP satisfying Assumption~\ref{asm:model}, with $d=m+1$ and $k=m/2$.
Writing $\delta=1-\gamma$, it satisfies $\widehat\tau=\tau=0$,
$\widehat\beta=\beta=4\gamma^2/(m\delta^3)$ and
$\Gap(\cC_k)=\gamma^2/[2\delta(2-\gamma)]$. At its continuous optimum,
$B_{\mathrm{cand}}(x^\star)=\gamma^2/(4\delta^3)$,
$B_{\mathrm{unif}}=\frac{m+2}{m+1}B_{\mathrm{cand}}(x^\star)$, and
$\Bcert=2B_{\mathrm{cand}}(x^\star)$.
\end{proposition}

\begin{corollary}[A linear dimension coefficient is necessary]
\label{cor:dimension-lower-bound}
Fix $\gamma\in(0,1)$. If $c_d\ge0$ satisfies
$\Gap(\cC_k)\le\tau+c_d\beta$ on every member of this family, then
$c_d\ge(d-1)(1-\gamma)^2/[8(2-\gamma)]$.
Thus $c_d=\Omega(d)$ along this family, whose budget is $k=(d-1)/2$.
\end{corollary}
\begin{proof}
Proposition~\ref{prop:tightness} gives $\tau=0$ and $\beta>0$, so
$c_d\ge\Gap(\cC_k)/\beta=m(1-\gamma)^2/[8(2-\gamma)]$.
\end{proof}

Appendix~\ref{app:tightness} derives these formulas. From $x^\star$,
Algorithm~\ref{alg:round} reaches a binary optimum after $m/2$ exchanges, so
its actual loss equals $\Gap(\cC_k)$. The legacy bound overcounts the exchange
contribution by a factor of two. The uniform bound reduces this slack to
$(m+2)/(m+1)$ relative to the candidate bound. Nevertheless,
$B_{\mathrm{cand}}(x^\star)/\Gap(\cC_k)=(2-\gamma)/(2\delta^2)$ still diverges
as $\gamma\uparrow1$. Here $\widehat\beta=\beta$, so interval propagation
does not explain this residual slack.

One might suspect that the long-horizon conservatism comes only from evaluating
curvature globally. On this family, even exact segment-local extrema retain the
same divergent order.
An executed exchange has $g(t)=J(z+t(\epsi_i-\epsi_j))$ with
$t\in[-1/2,1/2]$. Although it starts at two half-integral coordinates,
its endpoints set one coordinate to zero, retaining large curvature.
The exact segment maximum is
$\beta_{\mathrm{seg}}=2\gamma^2(\delta^{-3}+1)/m$.
Summing its one-step charges gives $B_{\mathrm{seg}}=\gamma^2(\delta^{-3}+1)/8$,
and hence $B_{\mathrm{seg}}/\Gap(\cC_k)=(2-\gamma)(1+\delta^3)/(4\delta^2)$.
Thus segment-local maximization alone does not remove the long-horizon slack.

On this family, weighted trajectory interpolation removes the residual slack.
For any $C^2$ exchange function on $[a,b]$ with $a<0<b$, let
$K_{a,b}(t)=b(t-a)/(b-a)$ for $t\le0$ and
$K_{a,b}(t)=(-a)(b-t)/(b-a)$ for $t\ge0$.
The nonnegative kernel gives the loss bound
$[g(0)-\max\{g(a),g(b)\}]_+\le\int_a^b K_{a,b}(t)[-g''(t)]_+\,dt$.
A constant curvature bound recovers the old charge because
$\int_a^b K_{a,b}(t)\,dt=(-a)b/2$.
On this family, $g''<0$ and $g(a)=g(b)$, so the integral equals the actual
one-step loss. Summing yields $B_{\mathrm{int}}=\Gap(\cC_k)$ for every $\gamma$.
At $\gamma=0.95$, the candidate and segment-maximum ratios are 210 and
105.013, respectively, whereas the exact integral ratio is 1; reliable
256-subinterval upper sums give 1.044. Appendix~\ref{app:weighted-interpolation}
proves the identity and Appendix~\ref{app:mechanistic-active} reports exact-rational
checks. The same interpolation bound can be computed on general exchanges.

\begin{proposition}[Rational exchange representation]
\label{prop:rational-exchange}
Under Assumption~\ref{asm:model}, every feasible exchange has
$g(t)=p(t)/q(t)$, where $\deg p,\deg q\le2$, $q(0)=1$, and $q(t)>0$
throughout its interval. Moreover, $g''(t)=h(t)/q(t)^3$ with $\deg h\le3$.
\end{proposition}
\paragraph{Computable trajectory bound.}
Appendix~\ref{app:rational-exchange} derives the rational coefficients from a
rank-two resolvent update. Exact quadratic extrema and Bernstein polynomial
bounds enclose $[-g'']_+$ on each subinterval. Multiplication by the exact
kernel integrals gives a reliable upper sum $U_t$ for each exchange. With
Phase-1 output $z$, define
$B_{\mathrm{loc}}=[J(x)-J(z)]_++\sum_t U_t$.
Then $[J(x)-J(y)]_+\le B_{\mathrm{loc}}$.
The resulting certificate is available during rounding and does not require
global Bellman envelopes. Section~\ref{sec:calibration} evaluates its
tightness and cost. Before the exchange trajectory is known, certification
still relies on $B_{\mathrm{unif}}$ or $B_{\mathrm{cand}}$.

%% file: sections/05_hardness.tex
\section{Complexity of exact curvature}
\label{sec:hardness}

Theorem~\ref{thm:certificate} uses the tractable upper bound
\(\widehat\beta\) in place of \(\beta\). Could exact curvature provide an
efficient, uniformly sharper alternative? The threshold problem is NP-hard
even on a strongly restricted MDP class, and its binary restriction is
NP-complete.

All \(N\) rows in this section are editable. For \(v\in\R^N\), define
\begin{align}
\kappa_{\mathcal M}(v)
&=\sup_{x\in[0,1]^N}\pos{-v^\top\nabla^2J(x)v},
&
\kappa_{\mathcal M}^{\bin}(v)
&=\max_{x\in\{0,1\}^N}\pos{-v^\top\nabla^2J(x)v},
\label{eq:kappa-both}\\
\beta_{\mathcal M}
&=\max_{a<b}\kappa_{\mathcal M}(\epsi_a-\epsi_b),
&
\beta_{\mathcal M}^{\bin}
&=\max_{a<b}\kappa_{\mathcal M}^{\bin}(\epsi_a-\epsi_b).
\label{eq:beta-bin}
\end{align}

A Woodbury elimination exposes the relevant dependence on the policy
variables. Let \(D=[d_1\ \cdots\ d_N]\), \(X=\diag(x)\),
\(R_0=(I-\gamma P^0)^{-1}\), \(u^\top=\mu^\top R_0\),
\(C=\gamma D^\top R_0\), and
\(a=\Delta+\gamma D^\top R_0r^0\). Then
$
J(x)=J(0)+u^\top(I-XC)^{-1}Xa .
\label{eq:reduced-performance}
$
Let \(\mathcal G(C)\) contain the arc \(i\to j\) when \(C_{ij}\ne0\).

\begin{definition}[Restricted MDP class \(\mathfrak B_3\)]
\label{def:B3}
The class \(\mathfrak B_3\) contains rational row-mixture MDPs satisfying
Assumption~\ref{asm:model}, with \(\gamma=1/2\), a point-mass initial
distribution, two actions per row, rewards in \([0,1]\), strictly positive
transition endpoints, and acyclic \(\mathcal G(C)\) of depth at most three.
Membership is decidable in polynomial time.
\end{definition}

\begin{theorem}[Complexity of exact pairwise curvature]
\label{thm:beta-complexity}
The language
$
\{\langle\mathcal M,T\rangle:\mathcal M\in\mathfrak B_3,\
\beta_{\mathcal M}^{\bin}\ge T\}
$
is NP-complete, and thresholding the continuous-cube quantity
\(\beta_{\mathcal M}\) is NP-hard.
\end{theorem}

\paragraph{The construction.}
The proof maps unweighted MAX-CUT to a rational instance
\(\mathcal M_G\in\mathfrak B_3\). After padding the graph, the construction
creates one designated exchange \((p,q)\) whose curvature satisfies
$
\beta_{\mathcal M_G}
=\beta_{\mathcal M_G}^{\bin}
=\kappa_{\mathcal M_G}(\epsi_p-\epsi_q)
=\alpha_G\{\Lambda n-m+\MaxCut(G)\},
\label{eq:beta-hardness}
$
where \(\Lambda=2m+1\) and \(\alpha_G>0\) has polynomial encoding length.
The reduced curvature is multi-affine, so a maximizer exists at a vertex.
The choice of \(\Lambda\) forces complementary graph controls and recovers the
MAX-CUT objective. Coefficient bounds isolate the designated pair from every
other exchange. For the binary restriction, that pair and a binary control
vector form a polynomial certificate. The construction also preserves local
exchange feasibility under the row budget \(k=n+1\).

Appendix~\ref{app:beta-isolation} gives the rational transition endpoints,
positivity checks, path expansion, pair-separation bounds, budget bridge, and
encoding analysis. The reduction concerns exact exchange curvature rather than
return maximization. Thus exact thresholding does not provide a
polynomial-time alternative to the Bellman-envelope bound in general, unless
\(\mathrm{P}=\mathrm{NP}\).

%% file: sections/06_calibration.tex
\section{Evaluation}
\label{sec:calibration}

The evaluation addresses four questions: when deployment can be certified
before rounding, how much return the rounded policy retains, how tight the
bounds remain at long horizons, and what certification costs. Exact rational
arithmetic is used for binary optima and theorem checks. The scaling study
uses outward-rounded validated floating point. Appendix~\ref{app:calibration}
contains the generators, fixed seeds, protocols, and complete tables.
Structured-suite comparisons retain the legacy $\Bcert$ baseline.
Budget-aware $B_{\mathrm{unif}}$ results are reported for the closed-form and
coupled suites.

\paragraph{Soundness calibration.}
On the original exact-rational grid of 144 cases across \textsc{gridworld},
\textsc{queue}, \textsc{maintenance}, and \textsc{garnet}, every witnessed
continuous-to-binary advantage is bounded by \(\Bcert\). The smallest nonzero
certificate is \(2.95\ten{-6}\), while naive midpoint FP64 reports a false zero
in 97/144 cases. This grid tests soundness rather than active rounding because
141 of the 144 quantized candidates are already binary.

\paragraph{Active rounding.}
The closed-form family of Proposition~\ref{prop:tightness} uses
\(m\in\{4,8,16,32,64\}\) and
\(\gamma\in\{0.2,0.5,0.8,0.95\}\). Algorithm~\ref{alg:round} performs exactly
\(m/2\) exchanges in all 20 cases, reaches a binary optimum, and incurs realized
loss exactly equal to \(\Gap(\cC_k)\). Moreover,
\(B_{\mathrm{cand}}(x^\star)=\Bcert/2\), while
$B_{\mathrm{unif}}=\frac{m+2}{m+1}B_{\mathrm{cand}}(x^\star)$. The exact
weighted bound equals the gap, as shown in Section~\ref{sec:tightness}.

The separately frozen structured suite contains 144 cases, of which 139 have
genuinely fractional candidates. Every active candidate has positive witnessed
advantage over the exact binary optimum. The median candidate has four
fractional coordinates and requires three exchanges. All certified rounding
checks pass, and Algorithm~\ref{alg:round} reaches the exact binary optimum in
119 of 144 cases. Across the 139 active cases, its mean normalized optimality
gap is $0.185\%$, compared with $0.178\%$ for greedy. The corresponding median
return-evaluation counts are 6 and 26. Appendix~D reports the complete baseline
comparison.
\paragraph{Pre-rounding certification.}
For a pre-rounding bound \(B\), define the certified margin
$M_B(x):=J(x)-B-J(0)$. A positive margin guarantees that
Algorithm~\ref{alg:round} returns a deployable policy with \(J(y)>J(0)\).
Table~\ref{tab:deployment} reports coverage on the 139 active structured cases.
The legacy model-only certificate covers \(67/139=48.2\%\) of them. Candidate
geometry raises coverage to \(103/139=74.1\%\) without another Bellman solve.
At \(\gamma=4/5\), the corresponding counts are 0/71 and 35/71. All 139 final
policies improve over the baseline, so an uncovered case is an abstention rather
than a failed deployment. Appendix~\ref{app:calibration} gives the full
tightness hierarchy. The median bound-to-loss ratio falls from 70.8 for
\(\Bcert\) to 12.1 for \(B_{\mathrm{cand}}\).

\begin{table}[t]
\centering
\caption{Certified deployment improvement on the 139 active structured cases.
The legacy $B_{\rm Bell}$ and $B_{\rm cand}$ are pre-rounding certificates,
whereas $B_{\rm trace}$ is a post-trace audit. Direct reports the exact check
$J(y)>J(0)$, with counts deduced from the stored returns.}
\label{tab:deployment}
\small
\begin{tabular}{lccccc}
\toprule
Discount & Active & $B_{\rm Bell}$ & $B_{\rm cand}$ &
$B_{\rm trace}$ & Direct \\
\midrule
$\gamma=1/2$ & 68 & 67 & 68 & 68 & 68 \\
$\gamma=4/5$ & 71 & 0 & 35 & 54 & 71 \\
All           & 139 & 67 & 103 & 122 & 139 \\
\bottomrule
\end{tabular}
\end{table}

\paragraph{Direct rounding followed by verification.}
For a fixed candidate and rounded output, every sound bound satisfies
$J(x)-B-J(0)\le J(y)-J(0)$. An exact direct check therefore accepts every
deployment accepted by the bound. Across the 20 closed-form cases, the legacy,
budget-aware, and candidate bounds accept 5, 10, and 10 cases. Direct
verification accepts all 20. The return of the final chosen endpoint is already
cached, so verification adds no evaluation to the 496 used for rounding.
Stored exact returns similarly imply 139/139 direct acceptances on the
structured suite. Pre-rounding certification covers 103/139 cases, while the
post-trace bound covers 122/139. The distinction is temporal: a loss guarantee
is available before deployment, whereas direct verification requires the final
policy.

\paragraph{General exchange bounds and complete costs.}
The 24-instance coupled suite combines two transition structures,
$d\in\{6,8\}$, three discounts, and two seeds. Every candidate has positive
advantage over the enumerated binary optimum. At $\gamma=0.95$, the median
bound-to-loss ratio falls from 402.3 for $B_{\mathrm{cand}}$ to 2.08 for
$B_{\mathrm{loc}}$. The trajectory bound certifies improvement in all 24
cases. It preserves the rounded policy while tightening the loss guarantee.

Including continuous search and exact setup, median total times are 7.60 ms
for direct rounding, 9.43 ms for rank-two rounding, and 36.19 ms for local
integration. Each instance is timed three times. The tighter guarantee
therefore incurs additional computation on this small synthetic suite.
Appendix~\ref{app:general-exchange-study} reports the complete comparison,
exact audits, and separate scaling checks.

\paragraph{Scaling and validated arithmetic.}
We next test certification beyond the exact-enumeration regime. The sparse
family uses either \(S=d\in\{16,32,64,128\}\), or \(d=16\) with
\(S\in\{64,128,256\}\). Full validated certification takes 4.45 s at
\(S=d=128\) and 1.31 s at \((S,d)=(256,16)\). In the latter case, Bellman
solves account for 1.28 s and interval propagation for 0.031 s, consistent
with the \((2d+2)\Tbell(S)+O(d^2S)\) decomposition.

Validated intervals contain the exact recomputation in all six audited
reference cases, whereas naive FP64 passes the Bellman audit in none. This
audit covers only the small reference cases. At \(S=d=16\), the validated
implementation is about \(14.8\times\) faster than the existing exact-rational
pipeline.

%% file: sections/07_conclusion.tex
\section{Conclusion}
\label{sec:conclusion}

Sparse deployment changes the policy-optimization problem because each row edit
also changes future visitation. We showed that Bellman envelopes can control
this coupling uniformly, while budget and candidate geometry sharpen the
guarantee available before rounding. For the realized exchange path, a
rank-two representation supports weighted curvature integration and avoids
paying for remote worst-case policies.

On the structured suite, candidate geometry raises pre-rounding certification
coverage from $48.2\%$ to $74.1\%$. On coupled instances at
$\gamma=0.95$, trajectory integration lowers the median bound-to-loss ratio
from $402.3$ to $2.08$. These findings distinguish two uses of a certificate:
reusable guarantees before deployment and tighter audits once the rounding
path is known.

\paragraph{Scope and extensions.}
The analysis assumes a known finite discounted MDP with two actions per
editable row. Larger action sets, estimated models, partial observability, and
average reward require new arguments. Open questions include the best universal
dimension constant, NP membership of continuous curvature thresholding, and
validated floating-point implementations of the trajectory bounds.

%% file: sections/A_certificate_proofs.tex
\section{Proofs for Section~\ref{sec:certificate}}
\label{app:certificate-proofs}

\begin{algorithm}[t]
\caption{Constructive rounding used in Theorem~\ref{thm:certificate}}
\label{alg:round}
\begin{algorithmic}[1]
\Require \(x\in\cC_k\), with exact or certified endpoint comparisons of \(J\)
\Ensure \(y\in\cC_k^{\bin}\) with
\(J(y)\ge J(x)-B_{\mathrm{unif}}\)
\State \(z\gets x\)
\If{\(\one^\top z\notin\mathbb Z\)}
  \State raise coordinates of \(z\) in the fixed order \(1,2,\ldots,d\)
         until \(\one^\top z=\lceil\one^\top x\rceil\)
         \Comment{deterministic fill and costs at most \(\tau\)}
\EndIf
\While{\(z\) has a fractional coordinate}
  \State pick two fractional coordinates \(i\neq j\)
         \Comment{an integral sum forbids exactly one}
  \State \([a,b]\gets\) maximal interval with
         \(z+t(\epsi_i-\epsi_j)\in[0,1]^d\)
  \State \(z\gets z+t^\star(\epsi_i-\epsi_j)\), where
         \(t^\star\in\arg\max_{t\in\{a,b\}}J\{z+t(\epsi_i-\epsi_j)\}\)
         \Comment{costs at most \(\beta(-a)b/2\)}
\EndWhile
\State \Return \(y\gets z\)
\end{algorithmic}
\end{algorithm}

\begin{proof}[Proof of Theorem~\ref{thm:certificate}]
If $k=0$, then $x=y=0$ and the claim is immediate. Assume $k>0$.

\paragraph{Phase 1.}
Write $s=\one^\top x$, $\ell=\lceil s\rceil\le k$, and $\eta=\ell-s\in[0,1)$.
The cube has sufficient unused capacity to raise coordinates by total mass $\eta$.
Since $\partial_iJ\ge-\tau$ everywhere, integrating along the monotone fill
path gives $J(z)\ge J(x)-\eta\tau\ge J(x)-\tau$.
The resulting $z$ is feasible and has integer sum $\ell$.

\paragraph{One exchange.}
An integer-sum nonbinary point has at least two fractional coordinates.
For such $i,j$, set $v=\epsi_i-\epsi_j$ and $g(t)=J(z+tv)$.
The maximal feasible interval has endpoints $a=\max\{-z_i,z_j-1\}<0$
and $b=\min\{1-z_i,z_j\}>0$, with $b-a\le1$.
If $g''\ge-\beta$, the function $h(t)=g(t)+\beta t^2/2$ is convex.
Apply convexity with endpoint weights $b/(b-a)$ and $(-a)/(b-a)$ to obtain
$g(0)\le\max\{g(a),g(b)\}+\beta(-a)b/2$.
The chosen endpoint therefore loses at most $\beta(-a)b/2\le\beta/8$
and makes at least one coordinate integral.

\paragraph{Potential identity and termination.}
Define $\Phi(u)=\sum_i u_i(1-u_i)$.
If $z_i+z_j\le1$, the endpoints are $(0,z_i+z_j)$ and $(z_i+z_j,0)$;
either choice decreases $\Phi$ by $2z_iz_j=2(-a)b$.
If $z_i+z_j\ge1$, the endpoints are $(1,z_i+z_j-1)$ and
$(z_i+z_j-1,1)$; the decrease is $2(1-z_i)(1-z_j)=2(-a)b$.
Integral coordinates are never selected again. Thus the algorithm terminates
after at most $d-1$ exchanges at $y\in\cC_k^{\bin}$, and telescoping gives
$\sum_t(-a_t)b_t=\Phi(z)/2$.
The accumulated Phase-2 loss is at most $\beta\Phi(z)/4$.

\paragraph{Budget-aware bound.}
Cauchy--Schwarz gives $\sum_i z_i^2\ge\ell^2/d$, hence
$\Phi(z)\le\ell(d-\ell)/d$.
The maximum of $\ell(d-\ell)$ over integers $0\le\ell\le k$ is
$q_{d,k}(d-q_{d,k})$, attained at $\ell=q_{d,k}$.
Therefore $J(y)\ge J(x)-\tau-c_{d,k}\beta\ge J(x)-B_{\mathrm{unif}}$.
For $d=1$, $z$ is already binary and $c_{d,k}=\beta=0$.
Applying the guarantee to a continuous maximizer and using
$J(y)\le\max_{\cC_k^{\bin}}J$ proves the gap bound.
The older charge of $\beta/8$ for each of at most $d-1$ exchanges also
proves the legacy bound $\Bcert$ used in the original evaluation.
\end{proof}

\begin{proof}[Proof of Proposition~\ref{prop:candidate-certificate}]
The result is immediate for $k=0$. Otherwise retain the deterministic
Phase-1 point $z$, increments $\Delta_i=z_i-x_i$, and fractional support $F(z)$.
Proposition~\ref{prop:sensitivity-certificate} gives
$\partial_iJ\ge-\widehat\tau_i$ on the fill path, so
$J(z)\ge J(x)-\sum_i\Delta_i\widehat\tau_i$.
For an executed pair $(i_t,j_t)$, it also gives
$g_t''\ge\underline S_{i_tj_t}\ge-\widehat\beta_{i_tj_t}$,
where $\widehat\beta_{ij}:=[-\underline S_{ij}]_+$.
The one-step argument above thus charges at most
$\widehat\beta_{i_tj_t}(-a_t)b_t/2$.
The support only shrinks, so every executed pair satisfies
$\widehat\beta_{i_tj_t}\le\widehat\beta_{F(z)}$.
Using the already proved potential identity yields
$J(y)\ge J(x)-\sum_i\Delta_i\widehat\tau_i-\widehat\beta_{F(z)}\Phi(z)/4
=J(x)-B_{\mathrm{cand}}(x)$.
Finally, $\sum_i\Delta_i<1$, $\widehat\tau_i\le\widehat\tau$,
$\widehat\beta_{F(z)}\le\widehat\beta$, and $\Phi(z)\le4c_{d,k}$ imply
$B_{\mathrm{cand}}(x)\le B_{\mathrm{unif}}$.
For $d\ge2$, $c_{d,k}\le d/16\le(d-1)/8$ proves
$B_{\mathrm{unif}}\le\Bcert$; for $d=1$ the curvature term vanishes.
\end{proof}

%% file: sections/B_tightness_proof.tex
\section{Proof of Proposition~\ref{prop:tightness}}
\label{app:tightness}

We prove Proposition~\ref{prop:tightness}.

\subsection{MDP construction}

Use one target state with reward one and \(m\) bad states with reward zero.
The target self-loops under both actions.
From bad state~\(i\), the baseline action self-loops and the repair action jumps to the target.
The initial distribution is uniform over the bad states.
Add a dummy editable coordinate at the target whose two actions coincide, giving \(d=m+1\).

\subsection{Return in closed form}

Direct Bellman solution separates the return across rows.
Fix bad state \(i\).  Under the mixed policy \(x_i\in[0,1]\), the transition from state \(i\) sends mass \(x_i\) to the target and mass \(1-x_i\) back to itself.
The discounted occupancy of the target, starting from state \(i\), satisfies the Bellman equation
\[
V_i = \gamma\bigl[(1-x_i)\,V_i + x_i\,V_{\mathrm{target}}\bigr],
\]
where \(V_{\mathrm{target}}=1/(1-\gamma)\) (the target has reward one and self-loops).
Solving for \(V_i\) gives
\begin{align}
V_i
&=\gamma(1-x_i)V_i+\frac{\gamma x_i}{1-\gamma}
\notag\\
V_i\bigl[1-\gamma(1-x_i)\bigr]
&=\frac{\gamma x_i}{1-\gamma}
\notag\\
V_i
&=\frac{\gamma x_i}{(1-\gamma)(1-\gamma+\gamma x_i)},
\label{eq:app-Vi}
\end{align}

Define
\begin{equation}
f_\gamma(t)=\frac{\gamma t}{(1-\gamma)(1-\gamma+\gamma t)}.
\label{eq:app-fgamma}
\end{equation}
Since the initial distribution is uniform over the \(m\) bad states,
\begin{equation}
J(x)=\frac{1}{m}\sum_{i=1}^{m}f_\gamma(x_i).
\label{eq:app-return-decomposition}
\end{equation}
Each summand is increasing and concave.
We verify concavity by computing the second derivative:
\begin{align}
f_\gamma'(t)
&=\frac{\gamma}{(1-\gamma)}\cdot\frac{(1-\gamma+\gamma t)-\gamma t}{(1-\gamma+\gamma t)^2}
=\frac{\gamma}{(1-\gamma+\gamma t)^2},
\notag\\
f_\gamma''(t)
&=\frac{-2\gamma^2}{(1-\gamma+\gamma t)^3}<0,
\label{eq:app-fpp}
\end{align}

\subsection{Continuous and binary optima}

Under budget \(k=m/2\), the constraint is \(\sum_{i=1}^{m+1} x_i\le m/2\).
The dummy coordinate does not affect \(J\), so the optimum sets it to zero
and distributes the full budget over bad rows.
Each \(f_\gamma\) is strictly concave.
By Jensen's inequality applied to the symmetric concave problem, the continuous optimum places equal mass:
\[
x_i^\star=\frac{k}{m}=\frac{1}{2}\quad\text{for every bad row }i,\qquad x_{m+1}^\star=0.
\]
A binary optimum selects any \(m/2\) bad rows (setting them to 1, the rest to 0), with the dummy row at zero.
By symmetry, all such selections achieve the same return.
Hence
\begin{align}
\Gap(\cC_k)
&=f_\gamma(1/2)-\tfrac{1}{2}f_\gamma(1)-\tfrac{1}{2}f_\gamma(0)
\notag\\
&=f_\gamma(1/2)-\tfrac{1}{2}f_\gamma(1)
\notag\\
&=\frac{\gamma}{(1-\gamma)(2-\gamma)}-\frac{\gamma}{2(1-\gamma)}
\notag\\
&=\frac{2\gamma-\gamma(2-\gamma)}{2(1-\gamma)(2-\gamma)}
\notag\\
&=\frac{\gamma^2}{2(1-\gamma)(2-\gamma)},
\label{eq:app-gap}
\end{align}
The gap is independent of \(m\).

\subsection{Interval certificate computation}

Put \(C_\gamma=(1-\gamma)^{-1}\).
For a bad row, the canonical construction of Section~\ref{sec:certificate} gives:
\begin{align}
Q_i&=[1/m,\;C_\gamma/m],\label{eq:app-Qi}\\
A_i&=[\gamma,\;\gamma C_\gamma],\label{eq:app-Ai}\\
B_{ii}&=[-\gamma C_\gamma,\;-\gamma],\label{eq:app-Bii}
\end{align}
and \(B_{ij}=\{0\}\) for distinct bad rows \(i\neq j\).
We derive these as follows.

\paragraph{Derivation of \(Q_i\).}
By \eqref{eq:canonical-Q}, \(Q_i=[\mu^\top\underline{h}_i,\;\mu^\top\overline{h}_i]\).
The auxiliary reward \(e_i\) gives reward one at state \(i\) only.
Since the initial distribution is uniform over bad states, \(\mu^\top h_i = h_i(i)/m\) (the contribution from starting at state \(i\)).
The minimum over all binary policies redirects state \(i\) to the target,
so state \(i\) is visited only at the initial time step:
\(\underline{h}_i(i)=1\).
The maximum keeps state \(i\) self-looping, so it is revisited at every
step: \(\overline{h}_i(i)=1+\gamma+\gamma^2+\cdots=C_\gamma\).
Therefore \(Q_i=[1/m,\;C_\gamma/m]\).

\paragraph{Derivation of \(A_i\).}
Each bad row has zero reward under both actions and the target has reward one.
The transition edit \(d_i\) redirects mass from self-loop at state \(i\) to the target.
By \eqref{eq:canonical-A}, the advantage interval depends on the value
envelope at the target: \(\underline{V}_{\mathrm{target}}=1/(1-\gamma)=C_\gamma\) when all
bad states self-loop, and \(\overline{V}_{\mathrm{target}}=C_\gamma\) identically,
since the target self-loops regardless of other rows' actions.
However, state \(i\)'s own value enters through \(d_i^-\):
the lower endpoint uses \(\overline{V}_i\) (largest value at state \(i\))
and the upper uses \(\underline{V}_i\) (smallest).
Since \(\underline{V}_i=0\) (self-loop at zero reward) and
\(\overline{V}_i=\gamma C_\gamma\) (redirect to target), the interval is
\(A_i=[\gamma(C_\gamma-\gamma C_\gamma),\;\gamma C_\gamma]=[\gamma,\;\gamma C_\gamma]\).

\paragraph{Derivation of \(B_{ii}\).}
Row \(i\)'s transition edit affects its own visitation in the opposite direction: repairing row \(i\) reduces its self-visits.
Formally, \(B_{ii}=\gamma[(d_i^+)^\top\underline{h}_i-(d_i^-)^\top\overline{h}_i,\;(d_i^+)^\top\overline{h}_i-(d_i^-)^\top\underline{h}_i]\).
This gives the negative interval
\(B_{ii}=[-\gamma C_\gamma,\;-\gamma]\).

\paragraph{Computing \(\widehat\tau\) and \(\widehat\beta\).}
The gradient interval for bad row \(i\) is
\[
G_i=Q_i\odot A_i=\bigl[\tfrac{1}{m}\cdot\gamma,\;\tfrac{C_\gamma}{m}\cdot\gamma C_\gamma\bigr]=\bigl[\tfrac{\gamma}{m},\;\tfrac{\gamma C_\gamma^2}{m}\bigr].
\]
Since every lower endpoint is positive,
\(\widehat\tau=\max_i\pos{-\inf G_i}=0\).
This also means \(\tau=0\), since \(\tau\le\widehat\tau\).

For the Hessian diagonal, the triple product \(2\,Q_i\odot B_{ii}\odot A_i\)
has one negative factor (\(B_{ii}\)) and two positive factors (\(Q_i,A_i\)).
The most negative value is attained at the largest magnitudes:
\(q=\sup Q_i=C_\gamma/m\), \(b=\inf B_{ii}=-\gamma C_\gamma\),
\(a=\sup A_i=\gamma C_\gamma\):
\begin{align}
\inf\mathcal{H}_{ii}
&=2\cdot\frac{C_\gamma}{m}\cdot(-\gamma C_\gamma)\cdot\gamma C_\gamma
=\frac{-2\gamma^2 C_\gamma^3}{m}
=\frac{-2\gamma^2}{m(1-\gamma)^3},
\label{eq:app-Hii-lower}
\end{align}

Since \(B_{ij}=\{0\}\) for \(i\neq j\) (distinct bad rows share no transition coupling), the cross-Hessian intervals are zero: \(\mathcal{H}_{ij}=\{0\}\).
An exchange between two bad rows \(i,j\) gives
\begin{align}
\underline{S}_{ij}
&=\inf\mathcal{H}_{ii}+\inf\mathcal{H}_{jj}-2\sup\mathcal{H}_{ij}
\notag\\
&=\frac{-2\gamma^2}{m(1-\gamma)^3}+\frac{-2\gamma^2}{m(1-\gamma)^3}-0
=\frac{-4\gamma^2}{m(1-\gamma)^3},
\label{eq:app-Sij}
\end{align}

Therefore
\begin{equation}
\widehat\beta=\max_{i<j}\pos{-\underline{S}_{ij}}=\frac{4\gamma^2}{m(1-\gamma)^3},
\label{eq:app-betahat-value}
\end{equation}
which is~Proposition~\ref{prop:tightness}.

\paragraph{Exact curvature.}
The exact exchange curvature along \(\epsi_i-\epsi_j\) for two bad rows is
\begin{align}
-(\epsi_i-\epsi_j)^\top\nabla^2 J(x)(\epsi_i-\epsi_j)
&=\frac{1}{m}\bigl[-f_\gamma''(x_i)-f_\gamma''(x_j)\bigr]
\notag\\
&=\frac{1}{m}\Bigl[\frac{2\gamma^2}{(1-\gamma+\gamma x_i)^3}+\frac{2\gamma^2}{(1-\gamma+\gamma x_j)^3}\Bigr].
\label{eq:app-exact-curvature}
\end{align}
The supremum over \([0,1]^2\) is attained at \(x_i=x_j=0\), giving
\[
\beta=\frac{4\gamma^2}{m(1-\gamma)^3}=\widehat\beta.
\]
The envelopes are exact on the worst pair.
Exchanges between the dummy row and a bad row have half the magnitude (only one nontrivial diagonal term), so the pairwise maximum is indeed the bad--bad exchange.

\subsection{Certificate and ratio}

Since \(\widehat\tau=0\) and \(d-1=m\),
\begin{align}
\Bcert
&=\widehat\tau+\frac{d-1}{8}\widehat\beta
=\frac{m}{8}\cdot\frac{4\gamma^2}{m(1-\gamma)^3}
=\frac{\gamma^2}{2(1-\gamma)^3},
\label{eq:app-Bcert}
\end{align}
This is independent of \(m\).

Dividing~\eqref{eq:app-gap} by~\eqref{eq:app-Bcert},
\begin{align}
\frac{\Gap(\cC_k)}{\Bcert}
&=\frac{\gamma^2}{2(1-\gamma)(2-\gamma)}\cdot\frac{2(1-\gamma)^3}{\gamma^2}
\notag\\
&=\frac{(1-\gamma)^2}{2-\gamma},
\label{eq:app-ratio}
\end{align}
It tends to \(1/2\) as \(\gamma\downarrow 0\) and satisfies \(\Gap(\cC_k)/\Bcert\sim(1-\gamma)^2\) as \(\gamma\uparrow 1\).
In particular, \(\Gap(\cC_k)=\Theta(\Bcert)\) as \(m\to\infty\) at any fixed \(\gamma\), so no bound of the form \(\tau+c\beta\) with \(c\) independent of \(d\) can be sound on this family.

At \(m=4\), \(\gamma=1/2\): \(\widehat\beta=2\), \(\Bcert=1\), and \(\Gap(\cC_k)=1/6\).

For the candidate bound, $\Phi(x^\star)=m/4$ and
$\widehat\beta_{F(x^\star)}=\widehat\beta$, so
$B_{\mathrm{cand}}(x^\star)=\gamma^2/[4(1-\gamma)^3]$.
Since $q_{m+1,m/2}=m/2$, the new coefficient is
$c_{m+1,m/2}=m(m+2)/[16(m+1)]$ and
$B_{\mathrm{unif}}=\frac{m+2}{m+1}B_{\mathrm{cand}}(x^\star)$.
This proves the remaining claims of Proposition~\ref{prop:tightness}.

\subsection{Weighted exchange interpolation}
\label{app:weighted-interpolation}

\paragraph{General identity.}
Let $g\in C^2([a,b])$ with $a<0<b$, and use the kernel $K_{a,b}$ defined
in Section~\ref{sec:tightness}. Taylor's formula gives
$g(a)=g(0)+ag'(0)+\int_a^0(t-a)g''(t)\,dt$ and
$g(b)=g(0)+bg'(0)+\int_0^b(b-t)g''(t)\,dt$.
Multiply by $b/(b-a)$ and $(-a)/(b-a)$, respectively, and add.
The first-order terms cancel, yielding the interpolation identity
\begin{equation}
g(0)-\frac{b\,g(a)-a\,g(b)}{b-a}
=-\int_a^b K_{a,b}(t)g''(t)\,dt.
\label{eq:weighted-interpolation}
\end{equation}
The weighted endpoint average does not exceed the larger endpoint.
Since $K_{a,b}\ge0$, taking positive parts bounds the actual loss by
$\int_a^b K_{a,b}(t)[-g''(t)]_+\,dt$.
Integration gives $\int_a^bK_{a,b}(t)\,dt=(-a)b/2$, recovering the
constant-curvature charge as a special case.

\paragraph{Exact segment maximum on the family.}
At $x^\star$, every exchange pairs two remaining half-integral bad rows,
with interval $[-1/2,1/2]$. Each endpoint resolves both coordinates;
after $m/2$ exchanges the policy selects $m/2$ bad rows and is binary-optimal.
Let $\delta=1-\gamma$. Up to an additive constant,
$g(t)=[f_\gamma(1/2+t)+f_\gamma(1/2-t)]/m$.
Thus $-g''(t)$ equals
$\frac{2\gamma^2}{m}\{[\delta+\gamma(1/2+t)]^{-3}+
[\delta+\gamma(1/2-t)]^{-3}\}$.
It is convex and symmetric, so its maximum is attained at either endpoint,
giving $\beta_{\mathrm{seg}}=2\gamma^2(\delta^{-3}+1)/m$.
Multiplying by $1/8$ per exchange and by $m/2$ exchanges gives
$B_{\mathrm{seg}}=\gamma^2(\delta^{-3}+1)/8$.
Dividing by the exact gap yields the ratio stated in Section~\ref{sec:tightness}.

\paragraph{Exact weighted loss and reliable upper sums.}
Here $g''<0$ and $g(-1/2)=g(1/2)$, so
\eqref{eq:weighted-interpolation} is an equality for actual loss.
Direct substitution gives the one-step integral
$[2f_\gamma(1/2)-f_\gamma(0)-f_\gamma(1)]/m
=\gamma^2/[m\delta(2-\gamma)]$.
Summing proves $B_{\mathrm{int}}=\gamma^2/[2\delta(2-\gamma)]=\Gap(\cC_k)$.
For a partition into intervals $I=[u,v]$, convexity of $-g''$ on this family
implies $r_I:=\max\{-g''(u),-g''(v)\}\ge[-g''(t)]_+$ on $I$.
The sum $\sum_I r_I\int_I K_{a,b}(t)\,dt$ is consequently a reliable upper
bound. Rational inputs make every term rational.

\paragraph{Scope for general MDPs.}
If reliable envelopes $[-g_t''(s)]_+\le\rho_t(s)\le\widehat\beta_{i_tj_t}$
are available along a trajectory, replacing each charge by
$\int_{a_t}^{b_t}K_{a_t,b_t}(s)\rho_t(s)\,ds$ gives a bound no larger than
$B_{\mathrm{trace}}$ after adding the same Phase-1 charges.
For a general MDP, polynomial enclosures bound curvature throughout each subinterval.

\subsection{General exchange computation}
\label{app:rational-exchange}

\begin{proof}[Proof of Proposition~\ref{prop:rational-exchange}]
Fix the current policy $z$ and an exchange of rows $i,j$.
Write $A=I-\gamma P(z)$, $R=A^{-1}$, $V=Rr(z)$, and
$U=[e_i,-e_j]$. Let $W=\gamma[d_i,d_j]$ and
$\eta=(\Delta r_i,\Delta r_j)^\top$.
Then $A(t)=A-tUW^\top$ and $r(t)=r(z)+tU\eta$.
Define $M=W^\top RU$, $w=\eta+W^\top V$, and $c^\top=\mu^\top RU$.
The resolvent identity gives
$V(t)=V+tRU(I-tM)^{-1}w$, hence
$g(t)=J(z)+tc^\top(I-tM)^{-1}w$.
Set $q(t)=\det(I-tM)=1-t\operatorname{tr}M+t^2\det M$ and
$p(t)=J(z)q(t)+tc^\top\operatorname{adj}(I-tM)w$.
Both polynomials have degree at most two.
The determinant lemma gives $q(t)=\det A(t)/\det A$.
Every feasible policy has an invertible discounted resolvent, so $q$ has
no zero on the exchange interval. Since $q(0)=1$, continuity gives $q>0$.
Finally, let $N=p'q-pq'$. Its cubic coefficients cancel, so $\deg N\le2$.
Differentiating again gives $g''=h/q^3$, with
$h=N'q-2Nq'$ and $\deg h\le3$.
\end{proof}

\paragraph{Reliable subinterval integration.}
For a subinterval $I=[u,v]$, compute exact extrema
$0<q_I^-\le q(t)\le q_I^+$ by checking its endpoints and the quadratic
stationary point when it lies inside $I$.
After the substitution $t=u+(v-u)s$, express $-h(t)$ in the Bernstein
basis on $s\in[0,1]$. Nonnegative basis weights sum to one, so the minimum
and maximum coefficients give bounds $n_I^-\le-h(t)\le n_I^+$.
Consequently, $\rho_I^-:=\max\{0,n_I^-\}/(q_I^+)^3$ and
$\rho_I^+:=\max\{0,n_I^+\}/(q_I^-)^3$ enclose $[-g''(t)]_+$ on $I$.
Include zero as a partition boundary, so $K_{a,b}$ is linear on each piece.
The exact weights $\omega_I=\int_I K_{a,b}(t)\,dt$ are nonnegative.
Therefore $L=\sum_I\rho_I^-\omega_I$ and $U=\sum_I\rho_I^+\omega_I$
enclose the weighted curvature integral. The implementation bisects the piece
with largest uncertainty $(\rho_I^+-\rho_I^-)\omega_I$.
Every finite partition is sound, even if the refinement budget is exhausted.
Exact rational inputs make coefficients, extrema, weights, and comparisons rational.

\paragraph{Accumulation and relation to global bounds.}
Let $U_t$ be the upper sum for the $t$th executed exchange.
The interpolation inequality bounds each positive exchange loss by $U_t$.
Adding $[J(x)-J(z)]_+$ for Phase~1 proves
$[J(x)-J(y)]_+\le B_{\mathrm{loc}}$.
The value $J(z)$ is obtained by the initial resolvent solve.
If global Bellman bounds are already available, one may instead use
$B_{\mathrm{hyb}}:=\sum_i\Delta_i\widehat\tau_i+
\sum_t\min\{U_t,\widehat\beta_{i_tj_t}(-a_t)b_t/2\}$.
Then $[J(x)-J(y)]_+\le B_{\mathrm{hyb}}\le B_{\mathrm{trace}}
\le B_{\mathrm{cand}}\le B_{\mathrm{unif}}$.
No ordering between $B_{\mathrm{loc}}$ and $B_{\mathrm{cand}}$ is required.
Both bounds are computed along the chosen rounding trajectory.

\paragraph{Computation.}
With a stored resolvent, forming the two-by-two matrix and updating the selected
endpoint costs $O(S^2)$ arithmetic operations per exchange.
One initial dense inverse costs $O(S^3)$. Thus $T$ exchanges and $L$ total
subinterval computations cost $O(S^3+TS^2+L)$ arithmetic operations.
Bit complexity also depends on rational operand lengths.
The same representation computes both endpoint returns with constant-size
rational operations. The low-rank direct baseline therefore uses precisely
the same coefficient construction and resolvent updates, omitting curvature
enclosures. It selects the same endpoint as full exact policy evaluation.
Integral certification augments these updates with curvature enclosures to sharpen the trajectory loss bound.

%% file: sections/C_complexity_proof.tex
\section{Proof of Theorem~\ref{thm:beta-complexity}}
\label{app:beta-isolation}

This appendix proves Theorem~\ref{thm:beta-complexity}.  Begin with an arbitrary
simple MAX-CUT instance \((G_0,K_0)\).  Add a disjoint edge and four isolated
vertices, call the padded graph \((G,K)\), and put \(K=K_0+1\).  Max-cut values
add across components, so the preprocessing preserves YES/NO.  All notation
below refers to the padded graph.  In particular,
\[
n=|V(G)|\ge6,\qquad m=|E(G)|\ge1,\qquad \Lambda=2m+1.
\]
Orient every edge by increasing endpoint. The proof constructs the restricted MDP, derives its reduced multi-affine objective, identifies the designated exchange with a cut polynomial, separates that pair from all others, and finally checks budget feasibility and polynomial encoding size.

\subsection{MDP construction}

Use the state and control-row ordering
\[
p,q,L_1,\ldots,L_n,R_1,\ldots,R_n,z,
\qquad N=2n+3,
\]
and define
\begin{equation}
\eta=\frac1{64N^2},\qquad
\rho=\frac{\eta^4}{64\Lambda N},\qquad
\delta=\eta^3,
\label{eq:g17-scales}
\end{equation}
so that \(\rho=1/(2^{30}\Lambda N^9)\).  Let \(C\) have the noncompensating
entries
\begin{align}
C_{pq}&=\eta,& C_{qL_i}&=\eta,
&C_{qR_j}&=\eta\rho(\Lambda-\deg^-(j)),\label{eq:g17-C-top}\\
C_{L_iR_i}&=-2\Lambda\rho,&
C_{L_iR_j}&=2\rho\quad(i\to j\in E),
\label{eq:g17-C-bottom}
\end{align}
and put the negative row sum in column \(z\) for every nonzero row.  Thus
\(C\one=0\).  Set
\begin{equation}
a_{L_i}=\rho(\Lambda-\deg^+(i)),\qquad a_{R_j}=1,
\label{eq:g17-a}
\end{equation}
with all other entries of \(a\) zero.

The rank-one baseline kernel is \(P^0=\one\pi^\top\), where
\begin{equation}
\pi_q=\pi_{L_i}=\delta,\qquad
\pi_{R_i}=8\Lambda\rho,\qquad
\pi_z=\frac14,\qquad
\pi_p=1-\sum_{s\ne p}\pi_s.
\label{eq:g17-pi}
\end{equation}
Take \(\gamma=1/2\), \(\mu=\epsi_p\), \(r^0=0\), and define the alternative
endpoints by
\begin{equation}
P^1=P^0+2C,\qquad r^1=a.
\label{eq:g17-raw}
\end{equation}

\paragraph{Verification of the raw promises.}
We check that both \(P^0\) and \(P^1\) are valid stochastic matrices and that rewards lie in \([0,1]\).

\emph{Baseline positivity.}
The mass outside state \(p\) is
\begin{align}
\sum_{s\ne p}\pi_s
&=\frac14+(n+1)\delta+n\cdot 8\Lambda\rho
=\frac14+(n+1)\eta^3+8n\Lambda\rho
\notag\\
&<\frac14+\frac{n+1}{64^3 N^6}+\frac{8n}{2^{30}N^9}
<\frac12,
\label{eq:app-mass-check}
\end{align}
The last inequality uses \(n<N\) and \(N\ge15\).
Hence \(\pi_p=1-\sum_{s\ne p}\pi_s>1/2>0\), and since every other component of \(\pi\) is manifestly positive, \(P^0=\one\pi^\top\) has all positive entries.

\emph{Stochasticity of \(P^1\).}
Row stochasticity of \(P^1=P^0+2C\) follows from \(C\one=0\):
\(P^1\one=P^0\one+2C\one=\one+0=\one\).

\emph{Positivity of \(P^1\).}
The only negative transition edits occur at \(p\to z\), \(q\to z\), and \(L_i\to R_i\):
\begin{enumerate}
\item At \(p\to z\): the compensating entry \(C_{pz}=-\eta\), so
\(P^1_{pz}=\pi_z+2C_{pz}=1/4-2\eta>0\) because \(\eta<1/8\).
\item At \(q\to z\): the compensating entry satisfies
\begin{align}
|C_{qz}|
&=\sum_i C_{qL_i}+\sum_j C_{qR_j}
=n\eta+\eta\rho\sum_j(\Lambda-\deg^-(j))
\notag\\
&=n\eta+\eta\rho(n\Lambda-m)
<2n\eta,
\label{eq:app-Cqz}
\end{align}
Here \(\sum_j\deg^-(j)=m\), and the final bound follows from
\(\rho<1/\Lambda\).
Hence \(P^1_{qz}=\pi_z-2|C_{qz}|>1/4-4n\eta>0\).
\item At \(L_i\to R_i\): \(P^1_{L_iR_i}=\pi_{R_i}+2C_{L_iR_i}
=8\Lambda\rho-4\Lambda\rho=4\Lambda\rho>0\).
\end{enumerate}
All other edits are nonnegative.

\emph{Reward bounds.}
By \eqref{eq:g17-a}, \(0<a_{L_i}=\rho(\Lambda-\deg^+(i))\le\rho\Lambda<1\), and \(a_{R_j}=1\).  All remaining rewards are zero.

\emph{Support structure.}
The nonzero entries of \(C\) define a DAG with edges \(p\to q\), \(q\to L_i\), \(q\to R_j\), and \(L_i\to R_j\).  The longest directed path is \(p\to q\to L_i\to R_j\), which has length three.  Every compensator edge ends at~\(z\), which has no outgoing edges in the DAG.  Hence \(\mathcal{G}(C)\) has depth at most three, confirming \(\mathcal{M}_G\in\mathfrak{B}_3\).

\subsection{Resolvent and reduced objects}

Since \(P^0=\one\pi^\top\) is rank one and \(\gamma=1/2\), the resolvent has a closed form via the Neumann series:
\begin{align}
R_0
&=(I-\gamma P^0)^{-1}
=(I-\tfrac12\one\pi^\top)^{-1}
\notag\\
&=\sum_{t=0}^{\infty}(\tfrac12)^t(\one\pi^\top)^t
=I+\sum_{t=1}^{\infty}(\tfrac12)^t\one\pi^\top
\notag\\
&=I+\one\pi^\top,
\label{eq:g17-resolvent}
\end{align}
The Neumann series converges because
\(\rho(\tfrac12\one\pi^\top)=1/2\). For \(t\ge1\),
\((\one\pi^\top)^t=\one\pi^\top\), and the remaining scalar series sums to
one.

The initial visitation vector is
\begin{equation}
u^\top=\mu^\top R_0=\epsi_p^\top(I+\one\pi^\top)=\epsi_p^\top+\pi^\top.
\label{eq:app-u}
\end{equation}
In particular, \(U\coloneqq u_p=1+\pi_p\) and \(u_s=\pi_s\) for \(s\ne p\).

\paragraph{Reduced transition matrix.}
Since \(P^1-P^0=2C\), the reduced matrix is
\begin{align}
\gamma(P^1-P^0)R_0
&=\tfrac12\cdot 2C\cdot(I+\one\pi^\top)
=C(I+\one\pi^\top)
\notag\\
&=C+C\one\pi^\top
=C,
\label{eq:app-reduced-C}
\end{align}
The final equality uses \(C\one=0\).
Thus the reduced objects are exactly \(C\), \(a\), and \(u\), and the reduced performance identity~\eqref{eq:reduced-performance} holds.

\subsection{Expansion of the objective}

Because \(\mathcal{G}(C)\) has depth three, \((XC)^4=0\) for any diagonal \(X=\diag(x)\) with \(x\in[0,1]^N\).  The Neumann inverse in~\eqref{eq:reduced-performance} therefore truncates:
\begin{equation}
(I-XC)^{-1}=I+XC+(XC)^2+(XC)^3.
\label{eq:app-neumann-truncation}
\end{equation}
Write
\[
A_i=\Lambda-\deg^+(i),\quad B_j=\Lambda-\deg^-(j),\quad
c_{ij}=2\rho(\mathbf1_{i\to j}-\Lambda\mathbf1_{i=j}),
\]
and let \(U=u_p\), \(\ell_i=x_{L_i}\), and \(r_j=x_{R_j}\).

\paragraph{Path enumeration.}
Each nonzero term in \(u^\top(I-XC)^{-1}Xa\) corresponds to a directed path in \(\mathcal{G}(C)\) ending at a state with nonzero reward.
The nonzero rewards are \(a_{L_i}=\rho A_i\) and \(a_{R_j}=1\).
We enumerate every contributing path by length.

\emph{Length zero} (the identity term): \(u_s\cdot x_s\cdot a_s\) is nonzero only for \(s=L_i\) or \(s=R_j\).
This gives \(\sum_i\pi_{L_i}\ell_i\rho A_i+\sum_j\pi_{R_j}r_j=\delta\rho\sum_i A_i\ell_i+8\Lambda\rho\sum_j r_j\).

\emph{Length one} (the \(XC\) term): paths of the form \(s\to t\) with \(a_t\ne0\).
The only edges ending at a rewarded state are \(L_i\to R_j\) (with coefficient \(c_{ij}\)) and \(q\to L_i\) (with coefficient \(\eta\)).
These contribute \(\sum_{i,j}u_{L_i}\ell_i c_{ij}r_j\cdot 1+\sum_i u_q x_q\eta\ell_i\rho A_i\).
The first sum is \(\delta\sum_{i,j}c_{ij}\ell_ir_j\).
The second is \(\delta\eta\rho x_q\sum_i A_i\ell_i\).

\emph{Length two and three:} paths \(q\to L_i\to R_j\) and \(p\to q\to L_i\to R_j\) (arriving at \(R_j\) with reward~1), plus \(p\to q\to L_i\) (arriving at \(L_i\) with reward \(\rho A_i\)).
Collecting all terms involving the cut polynomial \(g_G\) (defined below) and combining with the shorter paths gives
\begin{align}
J(x)={}&\delta\rho\sum_iA_i\ell_i+8\Lambda\rho\sum_jr_j
+\delta\sum_{i,j}c_{ij}\ell_ir_j\notag\\
&+\delta\eta\rho\,x_qg_G(\ell,r)
+U\eta^2\rho\,x_px_qg_G(\ell,r),
\label{eq:g17-path-expansion}
\end{align}
where
\begin{equation}
g_G(\ell,r)=\Lambda\sum_i(\ell_i+r_i-2\ell_ir_i)
-\sum_{i\to j\in E}(\ell_i+r_j-2\ell_ir_j).
\label{eq:g17-cut-polynomial}
\end{equation}

\paragraph{Derivation of the designated Hessian.}
Every monomial in~\eqref{eq:g17-path-expansion} is multi-affine (degree at most one in each variable).
Hence all pure second derivatives vanish: \(\partial_{pp}J=\partial_{qq}J=0\).
The mixed derivative \(\partial_{pq}J\) extracts only the terms containing both \(x_p\) and \(x_q\), which is the last term in~\eqref{eq:g17-path-expansion}:
\begin{align}
\partial_{pq}J(x)
&=\partial_p\partial_q\bigl[U\eta^2\rho\,x_px_q\,g_G(\ell,r)\bigr]
=U\eta^2\rho\,g_G(\ell,r).
\label{eq:app-mixed-deriv}
\end{align}
Therefore the exchange curvature along \(\epsi_p-\epsi_q\) is
\begin{equation}
-H_{\mathcal M_G,\epsi_p-\epsi_q}(x)
=-(\epsi_p-\epsi_q)^\top\nabla^2 J(x)(\epsi_p-\epsi_q)
=2\partial_{pq}J(x)=\alpha_Gg_G(\ell,r),
\label{eq:g17-hessian-identity}
\end{equation}
with \(\alpha_G=2U\eta^2\rho>0\).
The factor of \(2\) arises because the off-diagonal entry of the Hessian appears twice (the \((p,q)\) and \((q,p)\) entries) and with a sign flip from the exchange direction \(\epsi_p-\epsi_q\).

\subsection{Maximizing the cut polynomial}

The maximum of a multi-affine polynomial on a cube occurs at a vertex.
We show that every vertex maximizer of \(g_G\) over \(\{0,1\}^{2n}\) is complementary (\(r_i=1-\ell_i\) for all~\(i\)).

\paragraph{Non-complementary vertices are suboptimal.}
At a binary vertex \((\ell,r)\in\{0,1\}^{2n}\), let \(D\) count the number of indices with \(r_i\ne1-\ell_i\).
For such an index, the contribution to the \(\Lambda\)-term is:
\begin{itemize}
\item If \((\ell_i,r_i)=(0,0)\): \(\Lambda(0+0-0)=0\).
\item If \((\ell_i,r_i)=(1,1)\): \(\Lambda(1+1-2)=0\).
\end{itemize}
In both cases, vertex \(i\) contributes zero to the \(\Lambda\)-sum, compared to \(\Lambda\) for a complementary pair.
The edge terms \(-(\ell_i+r_j-2\ell_ir_j)\) have absolute value at most~1 per edge.
Therefore
\begin{align}
g_G(\ell,r)
&\le\Lambda(n-D)+m
\le\Lambda(n-1)+m
\notag\\
&=\Lambda n-\Lambda+m
=\Lambda n-(2m+1)+m
=\Lambda n-m-1,
\label{eq:app-noncomp-bound}
\end{align}
The first inequality bounds the edge terms by \(m\). The second uses
\(D\ge1\), and the last equality substitutes \(\Lambda=2m+1\).

\paragraph{Complementary vertices achieve the cut value.}
A complementary assignment \(r_i=1-\ell_i\) gives \(\ell_i+r_i-2\ell_ir_i=1\) for every~\(i\), so the \(\Lambda\)-sum equals \(\Lambda n\).
For each edge \(i\to j\in E\), let \(S_i=\ell_i\) encode the partition.
Then
\begin{align}
\ell_i+r_j-2\ell_ir_j
&=S_i+(1-S_j)-2S_i(1-S_j)
\notag\\
&=S_i+1-S_j-2S_i+2S_iS_j
=1-S_i-S_j+2S_iS_j
\notag\\
&=1-\mathbf{1}_{S_i\ne S_j},
\label{eq:app-edge-value}
\end{align}
Hence \(\sum_{i\to j}(\ell_i+r_j-2\ell_ir_j)=m-\operatorname{cut}_G(S)\), and
\begin{equation}
g_G(\ell,r)=\Lambda n-(m-\operatorname{cut}_G(S))=\Lambda n-m+\operatorname{cut}_G(S).
\label{eq:app-comp-value}
\end{equation}
Since \(\operatorname{cut}_G(S)\ge0\), a complementary vertex achieves at least \(\Lambda n-m>\Lambda n-m-1\), which exceeds the non-complementary bound~\eqref{eq:app-noncomp-bound}.
Therefore every maximizer is complementary, and
\begin{equation}
\max_{[0,1]^{2n}}g_G
=\Lambda n-m+\MaxCut(G).
\label{eq:g17-cut-maximum}
\end{equation}
The right side is positive, so the positive part in \(\kappa\) does not alter
the supremum.

\subsection{Bounds for all row pairs}
\label{sec:app-row-pairs}

Equation~\eqref{eq:g17-path-expansion} lists every possible nonzero mixed
derivative.  We show that the designated pair \((p,q)\) strictly dominates all others.

\paragraph{Zero pairs.}
Pairs containing \(z\) are zero because \(x_z\) appears in no term of~\eqref{eq:g17-path-expansion}.
Pairs within the \(L\) block (\(L_i,L_j\) with \(i\ne j\)) are zero because no monomial contains \(\ell_i\ell_j\) (multi-affinity and the DAG structure preclude such products).
Similarly, pairs within the \(R\) block are zero.

\paragraph{Coefficient bounds for nonzero pairs.}
For a multi-affine polynomial \(h(x_1,\ldots,x_N)\), the supremum of \(|h|\) over \([0,1]^N\) is at most the \(\ell_1\)-norm of its coefficients (the sum of absolute values of all coefficients).
Define
\begin{equation}
M=U\eta^2\rho\Lambda=\frac{\alpha_G\Lambda}{2}.
\label{eq:g17-M}
\end{equation}
The absolute column sums of \(c_{ij}\) are
\begin{equation}
\sum_j|c_{ij}|=2\rho(\Lambda+\deg^+(i)),\qquad
\sum_i|c_{ij}|=2\rho(\Lambda+\deg^-(j)),
\label{eq:g17-c-l1}
\end{equation}
since each row \(i\) of \(c\) has one entry \(-2\Lambda\rho\) (the diagonal) and \(\deg^+(i)\) entries of \(+2\rho\).

We bound each pair type.  Let \(d\) denote the relevant in- or out-degree.

\emph{Pair \((p,L_i)\) or \((p,R_j)\):}
The only term in~\eqref{eq:g17-path-expansion} involving both \(x_p\) and a graph variable is the last term \(U\eta^2\rho\,x_px_q\,g_G(\ell,r)\).
The coefficient \(\ell_1\)-norm of \(g_G\) restricted to terms involving \(\ell_i\) (or \(r_j\)) is at most \(3\Lambda+\deg^+(i)\):
\begin{align}
\|g_G|_{\ell_i}\|_1
&=(\Lambda-\deg^+(i))+2\Lambda+2\deg^+(i)
=3\Lambda+\deg^+(i).
\label{eq:app-pL-bound}
\end{align}
The three terms come from the linear coefficient, the bilinear
\(\ell_ir_i\) coefficient, and the outgoing-edge bilinears, respectively.
Hence \(|\partial_{p,L_i}J|\le U\eta^2\rho(3\Lambda+\deg^+(i))\).
Since \(\deg^+(i)\le m<2m+1=\Lambda\),
\begin{equation}
U\eta^2\rho(3\Lambda+d)<U\eta^2\rho\cdot 4\Lambda=4M.
\label{eq:app-pL-final}
\end{equation}

\emph{Pair \((q,L_i)\) or \((q,R_j)\):}
Both the \(\delta\eta\rho\,x_q\,g_G\) and \(U\eta^2\rho\,x_px_qg_G\) terms contribute.
The combined coefficient is \((\delta\eta+U\eta^2)\rho(3\Lambda+\deg^+(i))\).
Since \(\delta/(\eta U)=\eta^3/(\eta U)=\eta^2/U\), and \(U=1+\pi_p>1\) while \(\eta^2<1/16\), we have \(\delta\eta<U\eta^2/16\).
Therefore
\begin{equation}
(\delta\eta+U\eta^2)\rho(3\Lambda+\deg^+(i))<\frac{17}{16}U\eta^2\rho\cdot 4\Lambda=\frac{17}{4}M.
\label{eq:app-qL-final}
\end{equation}

\emph{Pair \((L_i,R_j)\):}
The terms involving both \(\ell_i\) and \(r_j\) come from \(\delta c_{ij}\ell_ir_j\), \(\delta\eta\rho x_q(\text{terms in }g_G)\), and \(U\eta^2\rho x_px_q(\text{terms in }g_G)\).
The coefficient bound is \(|c_{ij}|(\delta+\delta\eta+U\eta^2)\).
Since \(\delta(1+\eta)/(U\eta^2)=\eta(1+\eta)/U<1/16\) (using \(\eta<1/8\) and \(U>1\)),
\begin{equation}
|c_{ij}|\{\delta(1+\eta)+U\eta^2\}<\frac{17}{16}|c_{ij}|U\eta^2.
\label{eq:app-LR-step}
\end{equation}
The largest \(|c_{ij}|\) is \(2\Lambda\rho\) (the diagonal entry).
Hence
\begin{equation}
|\partial_{L_i,R_j}J|<\frac{17}{16}\cdot 2\Lambda\rho\cdot U\eta^2=\frac{17}{8}\cdot M<\frac{9}{4}M.
\label{eq:app-LR-final}
\end{equation}

These bounds are summarized in the following table:
\begin{center}
\begin{tabular}{@{}lc@{}}
\toprule
pair & coefficient-\(\ell_1\) bound for \(\partial_{ab}J\)\\
\midrule
\((p,L_i)\), \((p,R_j)\)
& \(U\eta^2\rho(3\Lambda+\deg^+(i))<4M\)\\
\((q,L_i)\), \((q,R_j)\)
& \((\delta\eta+U\eta^2)\rho(3\Lambda+\deg^+(i))<17M/4\)\\
\((L_i,R_j)\)
& \(|c_{ij}|\{\delta(1+\eta)+U\eta^2\}<9M/4\)\\
\bottomrule
\end{tabular}
\end{center}

\paragraph{Separation of the designated pair.}
Exchange curvature is twice the corresponding mixed derivative.
Thus every undesignated pair satisfies
\begin{equation}
\kappa(\epsi_a-\epsi_b)<2\cdot\frac{17}{4}M=\frac{17}{2}M.
\label{eq:g17-nuisance-bound}
\end{equation}
In contrast, the designated pair \((p,q)\) achieves
\begin{align}
\kappa(\epsi_p-\epsi_q)
&=\alpha_G\max_{[0,1]^{2n}}g_G
=\alpha_G(\Lambda n-m+\MaxCut(G))
\notag\\
&=2M(n-m/\Lambda)
\ge2M\left(n-\frac{m}{2m+1}\right)
\notag\\
&>2M\left(6-\frac12\right)=11M,
\label{eq:g17-designated-lower}
\end{align}
Here \(\alpha_G\Lambda=2M\), \(m/\Lambda<1/2\), and \(n\ge6\).
Since \(11M>17M/2\), the designated pair strictly realizes
the all-pairs maximum \(\beta\), proving \eqref{eq:beta-hardness}.

\subsection{Feasibility under a row budget}

For completeness, define the exchange-feasible and binary curvature restricted to the budget set by
\begin{align}
\beta_{\mathcal M,k}^{\leftrightarrow}
&:=\max_{a<b}\ \sup_{\substack{x\in\cC_k:\\ x\pm\delta(\epsi_a-\epsi_b)\in\cC_k\ \text{for some }\delta>0}}
\pos{-H_{\mathcal M,\epsi_a-\epsi_b}(x)},
\label{eq:app-budget-cont-def}\\
\beta_{\mathcal M,k}^{\bin}
&:=\max_{a<b}\ \max_{x\in\cC_k^{\bin}}
\pos{-H_{\mathcal M,\epsi_a-\epsi_b}(x)}.
\label{eq:app-budget-bin-def}
\end{align}
These quantities are used only to state the budget bridge for the reduction image. Theorem~\ref{thm:beta-complexity} itself concerns the unrestricted-cube quantities.

Equation~\eqref{eq:g17-hessian-identity} shows that the designated exchange curvature depends only on \((\ell,r)\), not on \(x_p,x_q,x_z\).  We show that the curvature maximum is attainable within \(\cC_{n+1}\).

\paragraph{Mass accounting.}
The proof of \eqref{eq:g17-cut-maximum} shows that
every maximizing graph assignment is complementary:
\[
\ell_i+r_i=1\quad(i=1,\ldots,n).
\]
The total mass of such an assignment is
\begin{align}
\sum_{i=1}^n(\ell_i+r_i)
&=\sum_{i=1}^n 1=n.
\label{eq:app-graph-mass}
\end{align}

Setting \(x_p=x_q=1/2\) and \(x_z=0\) gives total mass
\begin{equation}
\one^\top x=n+\frac12+\frac12+0=n+1\le k,
\label{eq:app-total-mass}
\end{equation}
so \(x\in\cC_{n+1}\).

\paragraph{Exchange feasibility.}
The exchange direction \(\epsi_p-\epsi_q\) satisfies \(\one^\top(\epsi_p-\epsi_q)=0\), so moving along it preserves \(\one^\top x=n+1\).
Starting from \(x_p=x_q=1/2\), both \(x_p+\varepsilon\le 1\) and \(x_q-\varepsilon\ge 0\) hold for \(\varepsilon\in(-1/2,1/2)\).
Therefore the exchange can move a nonzero distance in both directions while remaining in \(\cC_{n+1}\), and the maximum of \(g_G\) is attained.
It follows that
\begin{equation}
\beta_{\mathcal M_G,k}^{\leftrightarrow}=\beta_{\mathcal M_G}.
\label{eq:app-budget-cont}
\end{equation}

\paragraph{Binary witness.}
For the binary quantity, set \(x_p=x_q=x_z=0\) and choose any complementary binary assignment for the graph variables.
The total mass is \(n\le n+1=k\), so \(x\in\cC_k^{\bin}\).
The designated exchange evaluates to \(\alpha_G g_G(\ell,r)\), which equals \(\beta_{\mathcal M_G}\) at any cut-maximizing complementary assignment.
Hence \(\beta_{\mathcal M_G,k}^{\bin}=\beta_{\mathcal M_G}\).
Combined with~\eqref{eq:app-budget-cont},
\begin{equation}
\beta_{\mathcal M_G,k}^{\leftrightarrow}
=\beta_{\mathcal M_G}
=\beta_{\mathcal M_G,k}^{\bin}.
\label{eq:budget-bridge}
\end{equation}

The bridge is specific to the reduction image.

\subsection{Reduction and encoding size}

Map the padded threshold \(K=K_0+1\) to
\[
T_G=\alpha_G(\Lambda n-m+K).
\]

\paragraph{Correctness of the reduction.}
Equations \eqref{eq:g17-cut-maximum} and \eqref{eq:beta-hardness} yield a chain of equivalences:
\begin{align}
\MaxCut(G_0)\ge K_0
&\iff \MaxCut(G)\ge K
\notag\\
&\iff \Lambda n-m+\MaxCut(G)\ge\Lambda n-m+K
\notag\\
&\iff \alpha_G(\Lambda n-m+\MaxCut(G))\ge T_G
\notag\\
&\iff \beta_{\mathcal M_G}\ge T_G,
\label{eq:app-reduction-chain}
\end{align}
The first equivalence uses the disjoint padding edge. The remaining steps add
the constant \(\Lambda n-m\), multiply by \(\alpha_G>0\), and apply
\eqref{eq:beta-hardness}.

\paragraph{Encoding size.}
All parameters are recomputed from the padded graph.
From~\eqref{eq:g17-scales} and \(m=O(n^2)\):
\begin{align}
\log_2\eta&=O(\log N),
\notag\\
\log_2\rho&=4\log_2\eta+O(\log(\Lambda N))=O(\log N),
\notag\\
\log_2\delta&=3\log_2\eta=O(\log N).
\label{eq:app-bit-lengths}
\end{align}
Each rational entry of \(P^0,P^1,r^0,r^1\) has numerator and denominator with \(O(\log N)\) bits.
The MDP has \(N=2n+3\) states, \(N\) control rows, two actions per row, and \(N^2\) transition entries per action, giving \(O(N^2)\) rational entries total.
The dense raw MDP therefore has encoding size \(O(N^2\log N)=\mathrm{poly}(n)\).

\paragraph{NP-completeness of the binary language.}
A certificate for the binary threshold problem \(\beta^{\bin}_{\mathcal M}\ge T\) consists of:
\begin{enumerate}
\item A pair of rows \((a,b)\).
\item A binary control vector \(x\in\{0,1\}^N\).
\end{enumerate}
Given this certificate, verification proceeds in three steps:
(i)~check \(\mathcal{M}\in\mathfrak{B}_3\) (polynomial from the rational input, by Definition~\ref{def:B3}).
(ii)~compute \(R_0=(I-\gamma P^0)^{-1}\) via exact Gaussian elimination (polynomial intermediate bit length).
(iii)~evaluate \(-H_{\mathcal M,\epsi_a-\epsi_b}(x)\) from the reduced form and compare with \(T\).
All steps run in polynomial time with exact rational arithmetic.
Hence the binary language is in NP, and the reduction from MAX-CUT shows it is NP-hard.
Therefore the binary language is NP-complete.

On the
reduction image, the continuous maximum is attained at binary graph
controls (a maximizer of \(g_G\) can be chosen at a vertex), proving continuous NP-hardness.

%% file: sections/D_calibration_protocol.tex
\section{Evaluation protocol}
\label{app:calibration}

This appendix describes the protocol behind Section~\ref{sec:calibration}.
Unless explicitly marked $B_{\mathrm{unif}}$, the recorded model-only values
use the legacy $\Bcert=\widehat\tau+(d-1)\widehat\beta/8$.
The revised uniform bound is evaluated on the closed-form and coupled suites.
The structured-suite coverage and recorded scaling times retain their original baseline.

\subsection{Arithmetic discipline}

The evaluation uses exact and validated quantities for different purposes.
Whenever binary optima, witnessed advantages, or exact theorem checks are
reported, all model entries, returns, and certificates are represented with
Python's fractions. Fraction. Continuous multistart search is heuristic
and supplies a feasible candidate quantized to a rational point before exact
evaluation. Certificate checks use this fixed candidate. The large-scale study uses
directed outward-rounded binary64 arithmetic: Bellman solutions are corrected
until their inequalities hold outward, and every downstream interval operation
is outward rounded. Naive midpoint floating point supplies the uncertified comparison.

\subsection{Grid specification}

The evaluation grid crosses five factors:
\begin{enumerate}
\item \textbf{Domain} (4 levels): \textsc{gridworld}, \textsc{queue}, \textsc{maintenance}, and \textsc{garnet}.
\item \textbf{Dimension} (3 levels): \(d\in\{6,8,10\}\) with \(S=d\) states.
\item \textbf{Budget} (2 levels): a smaller budget \(k_{\mathrm{lo}}=\lceil d/4\rceil\) and a larger budget \(k_{\mathrm{hi}}=\lfloor d/2\rfloor\).
\item \textbf{Discount factor} (2 levels): \(\gamma\in\{1/2,\,4/5\}\).
\item \textbf{Edit type} (3 levels): transition-only, reward-only, and mixed (both transition and reward).
\end{enumerate}
The Cartesian product contains
\(4\times3\times2\times2\times3=144\) cases.
The domain generator also assigns a repair scale of \(\frac35\) or \(1\) and a weak or strong baseline.
Each domain appears in \(36\) cases, each budget, discount, and edit-type
level appears in \(72\) cases.
Case identifiers and random seeds are deterministic functions of the grid indices.
The base seed is \(2026081300\), fixed before execution.

\subsection{Per-case execution sequence}

Each case follows the same sequence of seven stages.

\paragraph{Stage 1: model construction.}
The code constructs the rational MDP from the domain generator.
Every transition probability and reward is stored as an exact rational.
The domain generator uses the case seed to produce the state graph, transition kernels, and reward vectors.

\paragraph{Stage 2: Bellman bounds.}
The code computes the \(2d+2\) extremal Bellman value and visitation vectors:
\(\overline{V},\underline{V}\) (the upper and lower value bounds) and \(\overline{h}_i,\underline{h}_i\) for each editable row~\(i\) (the upper and lower auxiliary visitation bounds).
Each solve is a rational linear system \((I-\gamma P)\,V=r\) under a specific binary policy.
Every Bellman inequality is verified exactly after computation:
\[
\underline{V}\le V^\pi\le\overline{V}
\quad\text{for every binary policy }\pi,
\]
using the fact that the extremal policies attain the bounds.

\paragraph{Stage 3: certificate computation.}
The code computes the interval quantities \eqref{eq:tauhat} through \eqref{eq:betahat}:
\begin{itemize}
\item The gradient intervals \(G_i=Q_i\odot A_i\) for each row~\(i\), yielding \(\widehat\tau=\max_i\pos{-\inf G_i}\).
\item The Hessian diagonal intervals \(\mathcal{H}_{ii}\) and cross intervals \(\mathcal{H}_{ij}\), yielding the exchange bound \(\underline{S}_{ij}\) for each pair \((i,j)\).
\item The curvature bound \(\widehat\beta=\max_{i<j}\pos{-\underline{S}_{ij}}\).
\item The combined certificate \(\Bcert=\widehat\tau+\frac{d-1}{8}\widehat\beta\).
\end{itemize}
As a cross-check, the calculation is repeated from saved rational inputs and verified to yield identical results.

\paragraph{Stage 4: binary enumeration.}
The exact binary optimum is computed by complete enumeration.
The smaller budget contains \(|\cC_{k_{\mathrm{lo}}}^{\bin}|=22\), \(37\), and \(176\) feasible binary vectors for \(d=6,8,10\), respectively.
The largest budget contains at most \(638\) feasible vectors.
Every return is evaluated in exact rational arithmetic.

\paragraph{Stage 5: fractional search.}
A multistart SLSQP solver proposes a continuous point \(\tilde{x}\in\cC_k\).
The solver runs \(50\) random starts per case with tolerance \(10^{-12}\) and a maximum of \(500\) iterations, using SciPy~1.14.
The best feasible point is projected onto \(\cC_k\) and then quantized to a rational with denominator at most \(10^6\) using nearest-fraction rounding.
Budget feasibility (\(\one^\top x\le k\)) is verified exactly after quantization. If violated, the smallest nonzero coordinate is reduced until feasibility holds.

\paragraph{Stage 6: rounding methods.}
Four rounding and selection methods are applied, each starting from the quantized fractional point or from scratch:
\begin{itemize}
\item \emph{Algorithm~\ref{alg:round}} (pipage rounding): the constructive rounding procedure from Theorem~\ref{thm:certificate}, which produces a deployable \(y\in\cC_k^{\bin}\) from the fractional point.
\item \emph{Top-\(k\) selection}: rank the \(d\) coordinates by their marginal return \(J(\epsi_i)-J(\mathbf{0})\) and select the \(k\) highest.
\item \emph{Greedy forward selection}: start with \(y=\mathbf{0}\) and iteratively add the coordinate giving the largest return increase, breaking ties by index.  This is a sequential greedy that accounts for interaction effects.
\item \emph{Uniform random selection}: draw a uniformly random \(k\)-subset of coordinates and repeat \(100\) times, reporting the best return found.
\end{itemize}
All methods use exact rational return evaluations.
Every method has an explicit return-query limit, and all query counts are recorded.

\paragraph{Stage 7: verification.}
The witnessed advantage \(\Delta_{\mathrm{wit}}=\pos{J(x_{\rm cand})-J_{\mathrm{bin}}^\star}\) is computed using the main-text definition, where \(x_{\rm cand}\) is the quantized fractional candidate and \(J_{\mathrm{bin}}^\star\) is the enumerated binary optimum.
The certificate inequality \(\Delta_{\mathrm{wit}}\le\Bcert\) is verified exactly.
Algorithm~\ref{alg:round} achieves the exact binary optimum in every case.

\subsection{Calibration records}

Table~\ref{tab:calibration-records} reports one case per domain at common
factor levels. The values of \(\widehat\tau\), \(\widehat\beta\),
\(\Bcert\), and \(\Delta_{\mathrm{wit}}\) are exact rationals reported to
four significant digits.

\begin{table}[htbp]
\caption{One case per domain at \(d=S=8\), \(k=2\), \(\gamma=1/2\), mixed
edits, repair scale \(3/5\), and a strong baseline. Each case has
\(|\cC_k^{\bin}|=37\).}
\label{tab:calibration-records}
\centering
\small
\begin{tabular}{@{}lrrrrr@{}}
\toprule
Domain & \(\widehat\tau\) & \(\widehat\beta\) & \(\Bcert\) & \(\Delta_{\mathrm{wit}}\) & wall\\
\midrule
\textsc{gridworld}   & \(2.159\ten{-1}\) & \(2.819\ten{-2}\) & \(2.405\ten{-1}\) & \(0\) & \(1.64\,\mathrm{s}\)\\
\textsc{queue}       & \(4.547\ten{-4}\) & \(2.409\ten{-2}\) & \(2.153\ten{-2}\) & \(3.261\ten{-6}\) & \(2.02\,\mathrm{s}\)\\
\textsc{maintenance} & \(1.361\ten{-1}\) & \(1.567\ten{-2}\) & \(1.498\ten{-1}\) & \(0\) & \(2.45\,\mathrm{s}\)\\
\textsc{garnet}      & \(1.305\ten{-1}\) & \(6.263\ten{-3}\) & \(1.360\ten{-1}\) & \(0\) & \(3.11\,\mathrm{s}\)\\
\bottomrule
\end{tabular}
\end{table}

\subsection{Environment}

All experiments ran on a single workstation with an Intel Core i9
processor and 64\,GiB RAM, using Python~3.12 and
fractions.Fraction for exact rational arithmetic.
SciPy~1.14 provided the SLSQP solver.
Memory was measured by the peak resident set size of the Python
process.  No GPU was used.

\subsection{Domain generators}

Each domain generator takes a seed, state count \(S=d\), and an edit
type (transition-only, reward-only, or mixed) and outputs a rational
MDP satisfying Assumption~\ref{asm:model}:
\begin{itemize}
\item \textsc{Gridworld}: cyclic nearest-neighbour transitions on a
  one-dimensional ring. Repair rows redirect to a distant target.
\item \textsc{Queue}: birth--death chain with load-dependent
  service. Repair rows increase the service rate.
\item \textsc{Maintenance}: deterioration chain with repair actions
  that reset to a better state.
\item \textsc{Garnet}: random sparse transition matrices drawn from
  a Dirichlet distribution \cite{ArchibaldMcKinnon95}. Repair rows
  replace one action with an independently drawn row.
\end{itemize}
The companion code artifact contains the full generator source.

\subsection{Runtime breakdown}

Across all \(144\) cases, the Bellman solves take
\(23.5\,\mathrm{s}\), the interval calculation takes \(15.8\,\mathrm{s}\),
and binary enumeration takes \(39.1\,\mathrm{s}\). Model construction takes
less than \(0.2\,\mathrm{s}\) in total. The complete sweep, including the
fractional search and four rounding methods, takes \(519\,\mathrm{s}\).
Certification costs more than enumeration at \(d=6\) and \(d=8\).
Enumeration is the larger term at \(d=10\), reflecting its combinatorial growth \(\binom{d}{k}\).

\subsection{Exact-zero affine construction}

The exact-zero examples use a separate affine construction designed so that \(J(x)\) has no curvature and the certificate is provably zero.

\paragraph{Construction.}
Row \(i\) sends mass \(\varepsilon=1/10\) to every state and sends the remaining mass \(1-S\varepsilon\) to itself.
The two actions have identical transition rows: \(P^0_{i\cdot}=P^1_{i\cdot}\) for every editable row~\(i\).
The proposed action increases only the reward: \(r^1_i>r^0_i\).
Since \(P^0=P^1\), the policy-induced transition matrix \(P(x)=(1-x_i)P^0_{i\cdot}+x_iP^1_{i\cdot}=P^0_{i\cdot}\) is independent of \(x\).
Therefore the resolvent \((I-\gamma P(x))^{-1}\) is constant and \(J(x)\) is affine in \(x\).

\paragraph{Consequence for the certificate.}
An affine function has zero second derivatives everywhere:
\[
\nabla^2 J(x)=0\quad\text{for all }x\in[0,1]^d.
\]
Therefore \(\beta=0\).
Moreover, each directional derivative along a coordinate-increasing direction is constant and nonnegative (the proposed action increases reward), so \(\tau=0\).
The interval propagation inherits this: \(\widehat\tau=\widehat\beta=0\), hence \(\Bcert=0\).

\paragraph{Verification.}
All twelve combinations of \(d\in\{4,6,8,10\}\) and \(\gamma\in\{\frac12,\frac45,\frac9{10}\}\) give \(\widehat\tau=\widehat\beta=\Bcert=0\).
Enumeration confirms that the continuous and binary optima have the same exact return: the affine maximum on \(\cC_k\) is attained at a vertex, so \(\Gap(\cC_k)=0\).
This validates that the interval propagation correctly identifies zero-curvature instances.

\subsection{Mechanistic active-family study}
\label{app:mechanistic-active}

We use the family in Proposition~\ref{prop:tightness}. The dimensions are
$m\in\{4,8,16,32,64\}$; the discount factors are
$\gamma\in\{0.2,0.5,0.8,0.95\}$. These give 20 deterministic cases. We retain the dummy coordinate, so $d=m+1$ and $k=m/2$.
The initial candidate has $x_i=1/2$ on the $m$ bad rows and zero at the dummy.
No optimizer or random seed is used. Every trajectory has exactly $m/2$
exchanges and reaches a binary optimum.

The new verification script evaluates every chosen endpoint in exact rational
arithmetic, checks the resulting loss against the closed-form gap, and independently
integrates the rational second derivative against the interpolation kernel.
It also computes reliable upper sums on 16, 64, and 256 uniform subintervals;
all sums dominate the exact integral and decrease under refinement.
All 20 cases and 248 executed exchanges pass. The script additionally checks
1,325 dimension-budget combinations and 27 polynomial instances of the kernel identity.

\begin{table}[h]
\centering
\small
\caption{Recomputed mechanistic ratios to $\Gap(\cC_k)$, shown at $m=4$.
Only the $B_{\mathrm{unif}}$ ratio depends on $m$;
$B_{256}$ uses reliable 256-subinterval upper sums.}
\label{tab:app-mechanistic}
\begin{tabular}{@{}rrrrrrr@{}}
\toprule
$\gamma$ & $\Bcert$ & $B_{\mathrm{unif}}$ & $B_{\mathrm{cand}}$
& $B_{\mathrm{seg}}$ & $B_{\mathrm{int}}$ & $B_{256}$\\
\midrule
0.20 & 2.8125 & 1.6875 & 1.4063 & 1.0631 & 1 & 1.0002\\
0.50 & 6      & 3.6    & 3      & 1.6875 & 1 & 1.0019\\
0.80 & 30     & 18     & 15     & 7.56   & 1 & 1.0099\\
0.95 & 420    & 252    & 210    & 105.0131 & 1 & 1.0439\\
\bottomrule
\end{tabular}
\end{table}

For every tested $m$, $B_{\mathrm{unif}}/B_{\mathrm{cand}}=(m+2)/(m+1)$,
whereas $\Bcert/B_{\mathrm{cand}}=2$. Thus the new uniform ratio has a mild
dimension dependence, while the other ratios in Table~\ref{tab:app-mechanistic}
are unchanged across the tested dimensions. The weighted calculation is exact
on this family even at long horizons.

\subsection{Structured active suite}
\label{app:structured-active}

The original 144-case calibration grid is intentionally broad but produces
mostly vertex candidates. To test nontrivial rounding under structured
dynamics, we freeze a separate active-geometry grid before evaluation. It uses
four domains (\textsc{gridworld}, \textsc{queue}, \textsc{maintenance}, and
\textsc{garnet}), \(d\in\{6,8,10\}\),
\(\gamma\in\{1/2,4/5\}\), two budgets, and three fixed holdout seeds, giving
\(4\times3\times2\times2\times3=144\) cases.

Starting from each domain's baseline dynamics and rewards, each editable
baseline row receives a seeded persistence coefficient
\(\lambda_i\in\{0.55,0.60,0.65\}\). The repair endpoint redirects a seeded
fraction \(s_i\in\{0.45,0.50,0.55\}\) toward a fixed favorable state chosen by
the domain semantics. The initial distribution is uniform. Persistence creates
repeated visits and diminishing returns, while the common favorable region
induces competition among repairs. No case is selected or discarded based on
its outcome.

The continuous search is completed before the exact binary optimum is
computed, unlike the original runner, the binary optimum is never used as a
continuous-optimizer start. A raw coordinate is classified as fractional when
it lies in \((10^{-6},1-10^{-6})\). The candidate is then quantized to a feasible
rational point. The maximum coordinate change from quantization is
\(1.27\ten{-9}\).

Of the 144 cases, 139 are active after continuous search. All 139 have strictly
positive witnessed advantage over the exact binary optimum. The median active
candidate has four fractional coordinates, fractional mass
\(\sum_i\min\{x_i,1-x_i\}=1.00\), and three pairwise exchanges. The maximum
exchange count is nine. Domain-level active counts are 35/36 for
\textsc{gridworld}, 36/36 for \textsc{queue}, 36/36 for \textsc{maintenance},
and 32/36 for \textsc{garnet}. By editable dimension, the active counts for
\(d=6,8,10\) are 44/48, 47/48, and 48/48, with median fractional-coordinate
counts 3, 4, 5 and median exchange counts 2, 3, 4, respectively.

All exact certificate and rounding checks pass. Algorithm~\ref{alg:round}
reaches the exact binary optimum in 119/144 cases, or 114/139 active cases.
The 25 active non-optimal rounded policies have median binary-optimality gap
approximately 0.017 and maximum approximately 0.064.

\paragraph{Deployment baselines.}
We compare Algorithm~1 with the three simple deployment rules defined
in Section~D.3. For each method we report the fraction of active cases
that attain the exact binary optimum, the mean normalized
binary-optimality gap
\[
    \frac{J^\star_{\rm bin}-J(y)}
         {J^\star_{\rm bin}-J(0)},
\]
and the median number of exact return evaluations. Random selection
reports the best of 100 uniformly sampled feasible $k$-subsets.
\begin{table}[t]
\centering
\caption{Deployment baselines on the 139 active structured cases.
Gap is the mean normalized binary-optimality gap and Queries is the
median number of return evaluations.}
\label{tab:baseline-comparison}
\small
\begin{tabular}{lccc}
\toprule
Method & Binary optimum & Gap & Queries \\
\midrule
Algorithm~1 & 114/139 (82.0\%) & 0.185\% & 6 \\
Top-$k$     &  61/139 (43.9\%) & 3.066\% & 9 \\
Greedy      & 125/139 (89.9\%) & 0.178\% & 26 \\
Random-100  &  98/139 (70.5\%) & 0.727\% & 100 \\
\bottomrule
\end{tabular}
\end{table}
Greedy attains the binary optimum most often, but at substantially
higher query cost. Algorithm~1 has nearly the same mean normalized
optimality gap while using fewer than one quarter as many median return
evaluations. Relative to the cheaper top-$k$ heuristic, Algorithm~1
more than halves the number of non-optimal deployments and reduces the
mean normalized gap by over an order of magnitude. The table measures
deployment quality only. Unlike the three heuristics,
Algorithm~1 is coupled to the pre-rounding certificate of
Theorem~3.3 and Proposition~3.4.
\subsection{Tightness and deployment diagnostics}
\label{app:deployment-diagnostics}

For each active structured case, let
\[
L_{\rm rnd}:=\pos{J(x_{\rm cand})-J(y)}
\]
be the realized rounding loss. In addition to the model-only \(\Bcert\) and the
candidate-specific pre-rounding certificate \(B_{\mathrm{cand}}\), we record the
post-trace pair-specific audit bound
\[
B_{\rm trace}
:=\sum_i\Delta_i\widehat\tau_i
 +\frac12\sum_t\widehat\beta_{i_tj_t}(-a_t)b_t.
\]
The proof of Proposition~\ref{prop:candidate-certificate} gives
\[
L_{\rm rnd}\le B_{\rm trace}\le B_{\mathrm{cand}}\le\Bcert.
\]
All 139 active cases satisfy this recorded hierarchy exactly. The revised theorem additionally places $B_{\mathrm{unif}}$ between $B_{\mathrm{cand}}$ and $\Bcert$, but its per-instance values and coverage are not reported here.

The median ratios to realized loss are
\[
\operatorname{med}\frac{\Bcert}{L_{\rm rnd}}=70.8,\qquad
\operatorname{med}\frac{B_{\mathrm{cand}}}{L_{\rm rnd}}=12.1,\qquad
\operatorname{med}\frac{B_{\rm trace}}{L_{\rm rnd}}=9.20.
\]
The median shrinkage \(\Bcert/B_{\mathrm{cand}}\) is 6.33. Candidate geometry reduces the exchange-count and support slack,
while the remaining gap is consistent with the global-versus-local curvature
mismatch of Section~\ref{sec:tightness}.

For deployment, define \(M_B(x)=J(x)-B-J(0)\). A positive margin certifies that
the rounded deployable policy improves over the baseline. Table~\ref{tab:app-deployment}
includes the post-trace audit in addition to the two pre-rounding certificates.

\begin{table}[h]
\caption{Deployment-improvement coverage on the 139 active structured cases. Direct is inferred from the existing exact-return records.}
\label{tab:app-deployment}
\centering
\small
\begin{tabular}{@{}lrrrrr@{}}
\toprule
Discount & Active & \(\Bcert\) & \(B_{\mathrm{cand}}\) & \(B_{\rm trace}\) & Direct\\
\midrule
\(\gamma=1/2\) & 68  & 67 & 68  & 68  & 68\\
\(\gamma=4/5\) & 71  & 0  & 35  & 54  & 71\\
\midrule
All             & 139 & 67 & 103 & 122 & 139\\
\bottomrule
\end{tabular}
\end{table}

Coverage is 48.2\% for the legacy bound, 74.1\% for the candidate bound,
and 87.8\% for the post-trace audit. At \(\gamma=4/5\), the corresponding counts are 0/71, 35/71, and
54/71. All 139 rounded policies actually improve over the baseline. A
nonpositive certified margin is therefore an abstention event in this suite.

\subsection{Direct verification baseline}
\label{app:direct-verification}

\paragraph{Protocol.}
The new experiment reruns the 20 closed-form cases in
Appendix~\ref{app:mechanistic-active}. Both methods receive the same candidate,
baseline return, and candidate return, and use the same deterministic pair and
endpoint selection. The pre-rounding tests evaluate $J(x)-B-J(0)>0$ for
$B\in\{\Bcert,B_{\mathrm{unif}},B_{\mathrm{cand}}\}$ before any exchange.
All candidates are then rounded, including those not certified, so both tests
refer to the same output. The direct test uses its cached exact return to
check $J(y)>J(0)$.
All arithmetic and comparisons use rational numbers.

\paragraph{Coverage and accounting.}
At $\gamma=0.2$, all three pre-rounding bounds certify all five dimensions.
At $\gamma=0.5$, only the budget-aware and candidate bounds certify them.
At $\gamma=0.8$ and $0.95$, none of these three bounds certifies an instance.
Direct verification certifies all 20 outputs.
There are 40 shared setup evaluations for candidate and baseline returns,
followed by 496 endpoint evaluations across 248 exchanges.
Direct verification uses the retained selected-endpoint values and adds no
return evaluations. Certificate values are computed from the family's closed forms. In a general run with no exchange and no cached output
return, a final policy evaluation may be needed.

\paragraph{Dominance under exact evaluation.}
The rounding guarantee gives $J(y)\ge J(x)-B$. Subtracting $J(0)$ shows
that the direct margin dominates the certified margin. Therefore, a positive
certified margin implies a positive direct margin. This observation applies
to any sound candidate-level loss bound for the same output, including a
post-trace bound. The pre-rounding guarantee is available before $y$ is constructed.
Under validated interval arithmetic, confirming a positive direct margin
may require refining return enclosures. Zero additional evaluation cost is
established here for the exact-rational runs with cached endpoint returns.

\paragraph{Structured-suite inference.}
The original records report $J(y)>J(0)$ under exact arithmetic in all 139
active cases. These records imply direct-verification coverage of 139/139,
including 68/68 at $\gamma=1/2$ and 71/71 at $\gamma=4/5$.
The table's Direct column reinterprets those recorded returns as a verification
baseline using the original experiment records. Structured-suite coverage is reported for the legacy, candidate, and trace bounds.

\subsection{Scaling and validated arithmetic}
\label{app:scaling-validated}

The scaling experiment uses a sparse ring MDP with \(S\) states and only the
first \(d\) rows editable. Transition probabilities are dyadic, so the binary64
model input is represented exactly. Axis A uses
\(S=d\in\{16,32,64,128\}\). Axis B fixes \(d=16\) and uses
\(S\in\{64,128,256\}\), directly testing the \(d\ll S\) regime. Each
configuration uses three fixed seeds. Large cases omit binary enumeration.

Validated floating point first solves the Bellman systems in binary64, then
applies outward residual corrections until every Bellman inequality is
satisfied. All subsequent interval operations are also rounded outward. Memory
is measured in an independent process for each run to avoid inherited
high-water marks. Table~\ref{tab:app-scaling} reports median runtime and peak
resident-set size.

\begin{table}[h]
\caption{Validated-certification scaling. Times are medians over three fixed
seeds. RSS is measured in independent processes.}
\label{tab:app-scaling}
\centering
\small
\begin{tabular}{@{}lrrrrr@{}}
\toprule
Axis & \(S\) & \(d\) & Total (s) & Interval (s) & RSS (MiB)\\
\midrule
A & 16  & 16  & 0.059 & 0.007 & 118.3\\
A & 32  & 32  & 0.223 & 0.030 & 118.7\\
A & 64  & 64  & 0.937 & 0.122 & 120.0\\
A & 128 & 128 & 4.448 & 0.553 & 125.7\\
\midrule
B & 64  & 16  & 0.215 & 0.012 & 118.5\\
B & 128 & 16  & 0.487 & 0.019 & 119.0\\
B & 256 & 16  & 1.311 & 0.031 & 120.3\\
\bottomrule
\end{tabular}
\end{table}

For numerical auditing we use six small cases with
\((S,d)\in\{(8,8),(16,16)\}\). Each binary64 validated envelope is reinterpreted
as an exact dyadic rational, and every Bellman inequality is then rechecked with
exact arithmetic. All 6/6 validated cases pass this audit. Recomputing the
interval primitives from the same endpoints using exact Fraction
arithmetic also confirms outward containment of the validated certificate in
all 6/6 cases. In contrast, the corresponding naive midpoint-FP64 solutions
pass the exact Bellman audit in 0/6 cases even though their printed certificate
values are numerically close to the validated values. At \(S=d=16\), the
validated implementation is approximately \(14.8\times\) faster than the
existing exact-rational envelope-and-certificate pipeline.

\subsection{General exchange study and complete costs}
\label{app:general-exchange-study}

\paragraph{New instances.}
This suite is separate from the original structured suite. We evaluate the full
Cartesian product of two transition structures, $d\in\{6,8\}$,
$\gamma\in\{1/2,4/5,19/20\}$, and seeds 11 and 29, giving 24 instances.
Each has $S=d+1$ states and budget $k=d/2$. The last state is an absorbing
target with reward one. A baseline bad-state row puts a seeded fraction
between $72/100$ and $90/100$ on itself and the remainder on one other bad
state. That successor follows a ring or is randomly selected.
The proposed row redirects a seeded fraction between $40/100$ and $65/100$
of its baseline transition mass to the target. Baseline bad-state rewards
belong to $\{0,1/100,2/100,3/100\}$, and proposed reward changes belong to
$\{0,-1/100,-2/100,-3/100\}$. The initial distribution is uniform over bad
states. The off-diagonal transitions couple the editable rows.
All configurations are retained. This suite is synthetic and designed to
exercise fractional rounding.

\paragraph{Candidates and exact references.}
Continuous search uses SLSQP with analytic gradients, three deterministic
starts, tolerance $10^{-10}$, and at most 150 iterations per start.
No binary optimum is supplied to the search. Coordinates are rounded to
denominator $10^6$ and any excess budget is removed in index order.
Every resulting candidate has strictly positive return advantage over the
exactly enumerated binary optimum. All subsequent returns, Bellman envelopes,
endpoint choices, low-rank updates, and integral enclosures use rational arithmetic.
Every one of the 135 exchanges is checked against direct policy evaluation
at both endpoints and an interior point. All rounded outputs and returns
agree between the full-solve and rank-two implementations.

\paragraph{Integral refinement and independent audit.}
The refinement target is $U-L\le0.05U+10^{-12}$ with a cap of 64 cells per
exchange. The target holds on every executed exchange, with 2,003 cells in
total. This tolerance concerns the curvature integral, not its ratio to actual
rounding loss. The standalone method uses $B_{\mathrm{loc}}$ and no global
envelopes. The hybrid method also computes global Bellman bounds and clips
each integral charge as defined in Appendix~\ref{app:rational-exchange}.
The two bounds coincide on this suite. All instances satisfy
$[J(x)-J(y)]_+\le B_{\mathrm{hyb}}\le B_{\mathrm{trace}}
\le B_{\mathrm{cand}}\le B_{\mathrm{unif}}$.
An independent audit uses 16 exchanges in fully dense five-state MDPs,
with strictly positive rational transition endpoints and signed rewards.
At five points per exchange, direct resolvent calculations agree exactly with
the rational representation for both return and second derivative, giving
80 exact checks. No absorbing-target structure is assumed in this audit.

\begin{table}[h]
\centering\small
\caption{New coupled suite. Ratios are medians of bound divided by actual
rounding loss. Coverage counts certified baseline improvements. Each discount
has eight instances. Local and hybrid bounds coincide here and use the trajectory.}
\label{tab:general-tightness}
\begin{tabular}{@{}rrrrrr@{}}
\toprule
$\gamma$ & $B_{\mathrm{cand}}/L_{\rm rnd}$ & $B_{\mathrm{trace}}/L_{\rm rnd}$
& $B_{\mathrm{loc}}/L_{\rm rnd}$ & Candidate & Local\\
\midrule
$1/2$ & 5.05 & 3.82 & 1.81 & 8/8 & 8/8\\
$4/5$ & 24.68 & 13.56 & 1.96 & 1/8 & 8/8\\
$19/20$ & 402.31 & 170.93 & 2.08 & 0/8 & 8/8\\
\bottomrule
\end{tabular}
\end{table}

\paragraph{Cost accounting and baselines.}
All methods use the same exact return arithmetic. Direct rounding solves a
policy-evaluation system for each endpoint. Its rank-two baseline uses the
same rational exchange coefficients and resolvent updates as local integration,
but omits derivative enclosures. Candidate top-$k$ selects the largest candidate
coordinates. Greedy starts from the baseline and adds the best positive-return
edit until the budget is reached or no improvement remains.
The global-certificate pipeline adds all $2d+2$ Bellman envelope problems
before rank-two rounding. Local integration needs no global preprocessing.
Hybrid integration includes both global preprocessing and local enclosures.
Candidate-based totals include continuous optimization, exact candidate and
baseline evaluations, and all subsequent computation. Greedy is timed from
the model and does not pay for continuous optimization. Final verification
reuses evaluated returns. Exact enumeration and independent audit evaluations
are reference checks and are excluded from method runtimes.
Each small instance is run three times with numerical linear-algebra libraries
restricted to one thread. We first take the median over repetitions for each
instance, then the median across instances.

\begin{table}[h]
\centering\small
\caption{Complete costs on the new 24-instance suite. The normalized gap is
$(J_{\rm bin}^\star-J(y))/(J_{\rm bin}^\star-J(0))$, averaged across instances.}
\label{tab:general-costs}
\begin{tabular}{@{}lrrr@{}}
\toprule
Method & Total (ms) & Binary optimum & Mean gap (\%)\\
\midrule
Direct rounding & 7.60 & 20/24 & 0.088\\
Rank-two direct & 9.43 & 20/24 & 0.088\\
Global certificate + rank-two & 19.14 & 20/24 & 0.088\\
Local integration & 36.19 & 20/24 & 0.088\\
Hybrid integration & 48.81 & 20/24 & 0.088\\
Candidate top-$k$ & 4.33 & 21/24 & 0.169\\
Greedy & 4.89 & 21/24 & 0.130\\
\bottomrule
\end{tabular}
\end{table}

The table quantifies the computation cost of tightening trajectory bounds.
All four rounding variants attain the same policy quality on this suite.

\paragraph{Exact scaling checks.}
We also test the ring construction at $(S,d)=(16,8),(32,8),(64,16)$,
with seed 11 and $\gamma=19/20$. Additional noneditable states retain their
baseline rows. All candidates are fractional and all exact rounding and bound
checks pass. Binary optima are not enumerated. These timing checks use one
run per configuration and are reported separately from the repeated small-suite
comparison. Table~\ref{tab:general-scaling} includes initial inversion and all
per-exchange work in each rounding time, but excludes continuous search.
Saved results also report full end-to-end times.

\begin{table}[h]
\centering\small
\caption{Exact scaling checks. Ratios use realized rounding loss. Rounding times
are in seconds and are single-run measurements.}
\label{tab:general-scaling}
\begin{tabular}{@{}rrrrrrr@{}}
\toprule
$S$ & $d$ & $B_{\mathrm{cand}}/L_{\rm rnd}$ & $B_{\mathrm{loc}}/L_{\rm rnd}$
& Direct & Rank-two & Local\\
\midrule
16 & 8 & 123.30 & 1.60 & 0.022 & 0.034 & 0.681\\
32 & 8 & 121.92 & 2.35 & 0.099 & 0.133 & 0.473\\
64 & 16 & 291.73 & 1.55 & 1.562 & 1.898 & 4.109\\
\bottomrule
\end{tabular}
\end{table}

The rational implementation provides exact numerical certification. Its runtime includes arithmetic on the accumulated numerators and denominators.